\documentclass{llncs}

\usepackage{amstext,amssymb,amsmath}

\usepackage{graphicx}
\usepackage[usenames,dvipsnames]{xcolor}
\usepackage{float}
\usepackage{enumerate}
\usepackage{bbm}
\usepackage{xspace}
\usepackage{pifont} 
\usepackage{misc} 
\usepackage{soul} 
\usepackage[normalem]{ulem} 
\usepackage{thmtools,thm-restate}
\usepackage{peterdefs} 
\usepackage{microtype}
\usepackage{tikz}
\usepackage{pgfplots, pgfplotstable}
\usepackage{boxedminipage}
\usepackage[sort,nocompress]{cite} 

\usepackage{url}

\usepackage{hyperref}
\hypersetup{
    bookmarks=true,         
    unicode=false,          
    pdftoolbar=true,        
    pdfmenubar=true,        
    pdffitwindow=false,     
    pdfstartview={FitH},    
    pdftitle={Computing FO-Rewritings in EL 
      in Practice: from Atomic to Conjunctive Queries},
    pdfauthor={Peter Hansen, Carsten Lutz},     
    pdfkeywords={Ontology-mediated Querying} {First-Order Rewritings} {EL} 
    {Conjunctive Queries},  
    pdfnewwindow=true,      
    colorlinks=true,        
    linkcolor=black,        
    citecolor=black,        
    urlcolor=black,         
}
\title{Computing FO-Rewritings in \EL in Practice:\\ from Atomic 
to Conjunctive Queries\thanks{This is a pre-print of an article published in ISWC2018. The final authenticated version is available online at: https://doi.org/10.1007/978-3-319-68288-4\_21}}
\date{}
\author{ 
Peter Hansen \and Carsten Lutz
}
\institute{
University of Bremen, Germany \\ 
\{hansen, clu\}@informatik.uni-bremen.de
}

\begin{document}
\maketitle

\begin{abstract}
  A prominent approach to implementing ontology-mediated queries
  (OMQs) is to rewrite into a first-order query, which is then
  executed using a conventional SQL database system. We consider the
  case where the ontology is formulated in the description logic \EL
  and the actual query is a conjunctive query and show that
  rewritings of such OMQs can be efficiently computed in practice, in
  a sound and complete way. Our approach combines a reduction with a
  decomposed backwards chaining algorithm for OMQs that are based on
  the simpler atomic queries, also illuminating the relationship
  between first-order rewritings of OMQs based on conjunctive and on
  atomic queries. Experiments with real-world ontologies show
  promising results.
\end{abstract}

\section{Introduction}

One of the most important tools in ontology-mediated
querying is \emph{query rewriting}: reformulate a given
ontology-mediated query (OMQ) in an equivalence-preserving way in a
query language that is supported by a database system used to store
the data.  Since SQL is the dominating query language in conventional
database systems, rewriting into SQL and into first-order logic (FO)
as its logical core has attracted particularly much attention
\cite{DLLiteJAR07,OurTODS14,OurICDT17,IJCAI13,IJCAI15,IJCAI16,TheFrench,PierisBarcelo17}.
In fact, the DL-Lite family of description logics (DLs) was invented
specifically with the aim to guarantee that FO-rewritings of OMQs
(whose ontology is formulated in DL-Lite) always exist
\cite{DLLiteJAR07,DLLiteJAIR09}, but is rather restricted in
expressive power. For essentially all other DLs, there are OMQs which
cannot be equivalently rewritten into an FO query. However, ontologies
used in real-world applications tend to have a very simple structure
and, consequently, one may hope that FO-rewritings of practically
relevant OMQs exist in the majority of cases. This hope was confirmed
in an experimental evaluation carried out in the context of the \EL
family of description logics where less than 1\kern0.07em\% of the considered
queries was found not to be FO-rewritable~\cite{IJCAI15}; moreover,
most of the negative cases seemed to be due to modeling mistakes in
the ontology.

In this paper, we focus on the description logic \EL, which can be
viewed as a logical core of the OWL EL profile of the OWL~2 ontology
language \cite{owlprofiles}.
We use $(\Lmc,\Qmc)$ to denote the OMQ language that consists of all
OMQs where the ontology is formulated in the description logic \Lmc
and the actual query is formulated in the query language
\Qmc. Important choices for \Qmc include \emph{atomic queries (AQs)}
and the much more expressive \emph{conjunctive queries (CQs)}.  It has
been shown in \cite{IJCAI13} that for OMQs from $(\EL,\text{AQ})$, it
is \ExpTime-complete to decide
FO-rewritability
. Combining the techniques from \cite{IJCAI13} and the backwards
chaining approach to query rewriting brought forward e.g.\ in
\cite{TheFrench,
  Backchase99}, a practical algorithm for computing FO-rewritings of
OMQs from $(\EL,\text{AQ})$ was then developed in \cite{IJCAI15}. This
algorithm is based on a \emph{decomposed version} of backwards
chaining that implements a form of structure sharing. It was
implemented in the \emph{Grind} system and shown to perform very well
in practice \cite{IJCAI15}. 
It is important to remark that the algorithm is
\emph{complete}, that is, it computes an FO-rewriting whenever there
is one and reports failure otherwise.

The aim of this paper is to devise a way to efficiently compute
FO-rewritings of OMQs from $(\EL,\text{CQ})$, and thus the challenge
is to deal with conjunctive queries instead of only with atomic
ones. Note that, as shown in \cite{IJCAI16}, FO-re\-wri\-tability in
$(\EL,\text{CQ})$ is still \ExpTime-complete. Our approach is to
combine a reduction with the decomposed algorithm from \cite{IJCAI15},
also illuminating the relationship between first-order rewritings of
OMQs based on CQs and on AQs. It is worthwhile to point out that naive
reductions of FO-rewritability in $(\EL,\text{CQ})$ to
FO-rewritability in $(\EL,\text{AQ})$ fail. In particular,
FO-rewritability of all AQs that occur in a CQ $q$ are neither a
sufficient nor a necessary condition for $q$ to be FO-rewritable.  As
a simple example, consider the OMQ that consists of the ontology and query 
$$
  \Omc = \{ \exists r . A \sqsubseteq A, \ \exists s . \top
  \sqsubseteq A\} \quad 
  \text{ and }
\quad q(x)=\exists y \, (A(x)
\wedge s(x,y))
$$
and which is FO-rewritable into $\exists y \, s(x,y)$, but the only AQ
$A(x)$ that occurs in $q$ is not FO-rewritable in the presence of
\Omc.\footnote{OMQs also allow to fix the signature (set of concept
  and role names) that can occur in the ABox. In this example, we do
  not assume any restriction on the ABox signature.} In fact, it is
not clear how to attain a reduction of FO-rewritability in
$(\EL,\text{CQ})$ to FO-rewritability in $(\EL,\text{AQ})$, and even
less so a polynomial time one. This leads us to considering mildly
restricted forms of CQs and admitting reductions that make certain
assumptions on the algorithm used to compute FO-rewritings in
$(\EL,\text{AQ})$---all of them are satisfied by the decomposed
backwards chaining algorithm implemented in Grind.

We first consider the class of \emph{tree-quantified CQs (tqCQs)} in
which the quantified parts of the CQ form a collection of directed
trees. In this case, we indeed achieve a polynomial time reduction to
FO-rewritability in $(\EL,\text{AQ})$. To also transfer actual
FO-rewritings from the OMQ constructed in the reduction to the
original OMQ, we make the assumption that the rewriting of the former
takes the form of a UCQ (union of conjunctive queries) in which every
CQ is tree-shaped and that, in a certain sense made precise in the
paper, atoms are never introduced into the rewriting `without a
reason'.  Both conditions are very natural in the context of backwards
chaining and satisfied by the decomposed algorithm.

We then move to \emph{rooted CQs (rCQs)} in which every quantified
variable must be reachable from some answer variable (in an undirected
sense, in the query graph). We consider this a mild restriction and
expect that almost all queries in practical applications will be rCQs.
In the rCQ case, we do not achieve a `black box' reduction. Instead, we
assume that FO-rewritings of the constructed OMQs from
$(\EL,\text{AQ})$ are obtained from a certain straightforward
backwards chaining algorithm or a refinement thereof as implemented in
the Grind system. We then show how to combine the construction of
(several) OMQs from $(\EL,\text{AQ})$, similar to those constructed in
the tqCQ case, with a modification of the assumed algorithm to decide
FO-rewritability in $(\EL,\text{rCQ})$ and to construct actual
rewritings. The approach involves exponential blowups, but only in
parameters that we expect to be very small in practical cases and
that, in particular, only depend on the actual query contained in the
OMQ but not on the ontology.

We have implemented our approach in the Grind system and carried
out experiments on five real-world ontologies with 10 hand-crafted
CQs for each. The average runtimes are between 0.5 and 19 seconds
(depending on the ontology), which we consider very reasonable
given that we are dealing with a complex static analysis problem. 

Proofs are deferred to the appendix, which is made available at
\url{http://www.cs.uni-bremen.de/tdki/research/papers.html}.

\medskip

{\bf Related Work.} We directly build on our prior work in
\cite{IJCAI15} as discussed above, and to a lesser degree also on
\cite{IJCAI13,IJCAI16}. The latter line of work has recently been
picked up in the context of existential rules
\cite{PierisBarcelo17}. The distinguishing
features of our work are that (1)~our algorithms are sound, complete,
and terminating, that is, they find an FO-rewriting if there is one
and report failure otherwise, and (2)~we rely on the decomposed
calculus from \cite{IJCAI15} that implements structure sharing for
constructing small rewritings and achieving practical feasibility. We
are not aware of other work that combines features (1) and (2) and is
applicable to OMQs based on \EL. In the context of the description
logic DL-Lite, though, the construction of small rewritings has
received a lot of attention, see e.g.\
\cite{GottlobKKPSZ14,KikotKZ12,Rodriguez-MuroC12,RosatiA10}. Producing
small rewritings of OMQs whose ontology is a set of existential rules
has been studied in \cite{KonigLM15}, but there are no termination
guarantees. Constructing small \emph{Datalog}-rewritings of OMQs based on
\EL, which are guaranteed to always exist, was studied e.g.\ in
\cite{EiterOSTX12,Perez-UrbinaMH10,StefanoniMH12,TrivelaSCS15}.  A
different approach to answering \EL-based OMQs using SQL databases is
the combined approach where the consequences of the ontology are
materialized in the data \cite{LutzTomanWolterIJCAI09,StefanoniEtAl}.


















\section{Preliminaries}
\label{sect:prelims}

Let \NC, \NR, and \NI be countably infinite sets of \emph{concept
  names}, \emph{role names}, and \emph{individual names}. An 
\emph{\EL-concept} is formed according to the syntax rule
$$
  C,D ::= \top \mid A \mid C \sqcap D \mid \exists r . C
$$
where $A$ ranges over \NC and $r$ over \NR. An \emph{\EL-TBox} \Tmc
is a finite set of \emph{concept inclusions} $C \sqsubseteq D$, with
$C$ and $D$ \EL-concepts. Throught the paper, we use \EL-TBoxes as
ontologies.
An \emph{ABox} is a finite set of
\emph{concept assertions} $A(a)$ and \emph{role assertions} $r(a,b)$
where $a$ and $b$ range over \NI.  We use 
$\mn{Ind}(\Amc)$ to denote the set of individual names in the ABox
\Amc. A \emph{signature} is a set of concept and role names. When an
ABox uses only symbols from a signature $\Sigma$, then we call it a
\emph{$\Sigma$-ABox}. To emphasize that a signature $\Sigma$ is used
to constrain the symbols admitted in ABoxes, we sometimes call
$\Sigma$ an \emph{ABox signature}.

The semantics of concepts, TBoxes, and ABoxes is defined in the usual
way, see \cite{DL-Textbook}. We write $\Tmc \models C \sqsubseteq D$
if the concept inclusion $C \sqsubseteq D$ is satisfied in every model
of \Tmc; when \Tmc is empty, we write $\models C \sqsubseteq
D$. 
As usual in ontology-mediated querying, we make the \emph{standard
  names assumption}, that is, an interpretation \Imc satisfies a
concept assertion $A(a)$ if $a \in A^\Imc$
 and a role assertion $r(a,b)$ if $(a,b) \in r^\Imc$.

A \emph{conjunctive query (CQ)} takes the form
  $q(\xbf)=\exists \ybf \, \vp (\xbf,\ybf)$ with $\xbf,\ybf$ tuples of
  variables and $\vp$ a conjunction of atoms of the form $A(x)$ and
  $r(x,y)$ that uses only variables from $\mn{var}(q) = \xbf \cup
  \ybf$. The variables $\xbf$ are the \emph{answer variables} of $q$,
  denoted $\mn{avar}(q)$, and the \emph{arity} of $q$ is the length of
  $\xbf$.
  Unless noted otherwise, we allow equality in CQs, but we assume
  w.l.o.g.\ that equality atoms contain only answer variables, and
  that when $x=y$ is an equality atom in $q$, then $y$ does not occur
  in any other atoms in $q$. Other occurences of equality can be
  eliminated by identifying variables. 
  An
  \emph{atomic query (AQ)} is a conjunctive query of the form $A(x)$.
  A \emph{union of conjunctive queries (UCQ)} is a disjunction of CQs
  that share the same answer variables.  

  An \emph{ontology-mediated query (OMQ)} is a triple
  $Q=(\Tmc\!,\Sigma,q)$ where \Tmc is a TBox, $\Sigma$ an ABox
  signature, and $q$ a CQ. We use $(\EL,\text{AQ})$ to denote the set
  of OMQs where \Tmc is an \EL-TBox and $q$ is an AQ, and similarly
  for $(\EL,\text{CQ})$ and so on.  We do generally not allow equality
  in CQs that are part of an OMQ. Let $Q=(\Tmc\!,\Sigma,q)$ be an
  OMQ, 
  \Amc a $\Sigma$-ABox and $\abf \subseteq \mn{Ind}(\Amc)$. We write
  $\Amc \models Q(\abf)$ if $\Imc \models q(\abf)$ for all models \Imc
  of \Tmc and \Amc. In this case, $\abf$ is a \emph{certain answer} to
  $Q$ on \Amc.

\begin{example}\label{ex:omq}
  Consider an example from the medical domain. The following ABox holds data 
  about patients and diagnoses:
  \[ \Amc = \{ \mn{Person}(a), \mn{hasDisease}(a, oca_1), 
  \mn{Albinism}(oca_1) \} \]
  A TBox $\Tmc_1$ is used to make domain knowledge available:
  \begin{align*} 
    \Tmc_1 = \{\mn{Albinism} &\sqsubseteq \mn{HereditaryDisease}, \\ 
    \mn{Person} \sqcap \exists \mn{hasDisease.HereditaryDisease} &\sqsubseteq 
    \mn{GeneticRiskPatient} \}
    \end{align*}
  Let $Q_1$ be the OMQ $(\Tmc_1, \Sigma_\mn{full}, q_1(x))$, where $q_1(x) = 
  \mn{GeneticRiskPatient}(x)$, and $\Sigma_\mn{full}$ contains all concept 
  and role names. It can be verified that $\Amc \models Q_1(a)$.
  \myEndEx
\end{example}
We do not distinguish between a CQ and the set of atoms in it and
associate with each CQ $q$ a directed graph $G_q :=
(\mn{var}(q),\{(x,y) \mid r(x,y) \in q
\})$ 
(equality atoms are not reflected).
A CQ $q$ is \emph{tree-shaped} if $G_q$ is a directed tree and
$r(x,y),s(x,y) \in q$ implies $r=s$. A \emph{tree CQ (tCQ)} is a
tree-shaped CQ with the root the only answer variable and a \emph{tree
  UCQ (tUCQ)} is a disjunction of tree CQs.  Every \EL-concept can be
viewed as a tree-shaped CQ and vice versa; for example, the
\EL-concept $A \sqcap \exists r . (B \sqcap \exists s . A)$
corresponds to the CQ $q(x) = \exists y,z \, A(x) \wedge r(x,y) \wedge
B(y) \wedge s(y,z) \wedge A(z)$. We will not always distinguish
between the two representations and even mix them. We might thus write
$\exists r.q$ to denote an \EL-concept when $q$ is a tree-shaped CQ;
if $q(x)$ is as in the example just given, then $\exists r . q$ is the
\EL-concept $\exists r . (A \sqcap \exists r . (B \sqcap \exists s
. A))$.  If convenient, we also view a CQ $q$ as an ABox $\Amc_q$
which is obtained from $q$ by dropping equality atoms and then
replacing each variable with an individual (not distinguishing answer
variables from quantified variables).  A \emph{rooted CQ (rCQ)} is a
CQ $q$ such that in the undirected graph induced by $G_q$, every
quantified variable is reachable from some answer variable.  A
\emph{tree-quantified CQ (tqCQ)} is an rCQ $q$ such that after
removing all atoms $r(x,y)$ with $x,y \in \mn{avar}(q)$, we obtain a
disjoint union of tCQs. We call these tCQs the \emph{tCQs in $q$}. For
example, $q(x_1,x_2)=\exists y_1,y_2 \, r(x_1,x_2) \wedge r(x_2,x_1)
\wedge r(x_1,y_1) \wedge s(x_2,y_2)$ is a tqCQ and the tCQs in $q$ are
$\exists y_1 \, r(x_1,y_1)$ and $\exists y_2 \, s(x_2,y_2)$; by adding
to $q$ the atom $r(y_1,y_2)$, we obtain an rCQ that is not a tqCQ.
%
%

An OMQ $Q=(\Tmc\!,\Sigma,q)$ is \emph{FO-rewritable} if there is a
first-order (FO) formula $\varphi$ such that $\Amc \models Q(\abf)$
iff $\Amc \models \varphi(\abf)$ for all $\Sigma$-ABoxes \Amc. In this
case, $\vp$ is an \emph{FO-rewriting of} $Q$. When $\vp$ happens to be
a UCQ, we speak of a \emph{UCQ-rewriting} and likewise for other
classes of queries. It is known that FO-rewritability coincides with
UCQ-rewritability for OMQs from $(\EL,\text{CQ})$
\cite{OurTODS14,IJCAI13}; note that equality is important here as, for
example, the OMQ $(\{B \sqsubseteq \exists r . A \}, \{B,r\},q)$ with
$q(x,y) = \exists z (r(x,z) \wedge r(y,z) \wedge A(z))$ rewrites into
the UCQ $q \vee (B(x) \wedge x=y)$, but not into an UCQ that does not
use equality.
\begin{example}\label{ex:fonoyes}
  We extend the TBox $\Tmc_1$  from Example~\ref{ex:omq} to additionally describe 
  the hereditary nature of genetic defects:
  \[ \Tmc_2 := \Tmc_1 \cup \{\mn{Person} \sqcap \exists \mn{hasParent.GeneticRiskPatient} 
  \sqsubseteq \mn{GeneticRiskPatient}\}. \]
  The OMQ $Q_1' = (\Tmc_2, \Sigma_\mn{full}, q_1(x))$ with $q_1(x)$ as in 
  Example~\ref{ex:omq}, is not FO-rewritable, intuitively
  because it expresses unbounded reachability along the
  $\mn{hasParent}$ role.  In contrast, consider the OMQ $Q_2 =
  (\Tmc_2, \Sigma_\mn{full}, q_2(x))$ where $q_2(x) = \exists y\;
  \mn{GeneticRiskPatient}(x) \wedge \mn{hasDisease}(x,y) \wedge
  \mn{Albinism}(y)$. Even though $q_2$ is an extension of $q_1$ with
  additional atoms, $Q_2$ is FO-rewritable, with $\varphi(x) = q_2(x)
  \vee \big(\exists y\; \mn{Person}(x) \wedge \mn{hasDisease}(x,y)
  \wedge \mn{Albinism}(y)\big)$ a concrete rewriting.
  \myEndEx
\end{example}

We shall sometimes refer to the problem of \emph{(query) containment}
between two OMQs $Q_1 = (\Tmc_1, \Sigma, q_1)$ and $Q_2 = (\Tmc_2,
\Sigma, q_2)$; we say \emph{$Q_1$ is contained in $Q_2$} if $\Amc
\models Q_1(\abf)$ implies $\Amc \models Q_2(\abf)$ for all
$\Sigma$-ABoxes \Amc and $\abf \subseteq \mn{Ind}(\Amc)$. If both OMQs are
from (\EL, rCQ) and $\Tmc_1 = \Tmc_2 = \Tmc\!$, then we denote this with
$q_1 \subseteq_{\Tmc} q_2$.


\smallskip

%
We now introduce two more technical notions that are central to the
constructions in Section~\ref{sect:rCQs}. Both notions have been used
before in the context of ontology-mediated querying, see for example
\cite{ijcar08,Lutz-DL08}. They are illustrated in
Example~\ref{ex:splitting} below.
%
%
\begin{definition}[Fork rewriting]
  Let $q_0$ be a CQ. \emph{Obtaining a CQ $q$ from $q_0$ by fork
  elimination} means to select two atoms $r(x_0,y)$ and
  $r(x_1,y)$ with $y$ an existentially quantified variable,
  then to replace 
every occurrence of $x_{1-i}$ in $q$ with $x_i$, where $i \in \{0,1\}$
  is chosen such that $x_i$ is an answer variable if any of $x_0,x_1$
  is an answer variable, and to finally add the atom $x_i = x_{1-i}$ if
  $x_{1-i}$ is an answer variable. When $q$ can be obtained from $q_0$
  by repeated (but not necessarily exhaustive) fork elimination, then
  $q$ is a \emph{fork rewriting} of~$q_0$.
\end{definition}
For a CQ $q$ and $V \subseteq \mn{var}(q)$, we use $q|_V$ to denote
the restriction of $q$ to the variables in~$V$, that is, $q|_V$ is the
set of atoms in $q$ that use only variables
from~$V$.

\begin{definition}[Splitting] \label{def:splitting} Let \Tmc be an
  \EL-TBox, $q$ a CQ, and \Amc an ABox. A \emph{splitting} of $q$
  w.r.t.\ $\Amc$ and \Tmc is a tuple $\Pi = \langle
  R,S_1,\dotsc,S_\ell,r_1,\dots,r_\ell,\mu,\nu \rangle$, where
  $R,S_1,\dotsc,S_n$ is a partitioning of $\mn{var}(q)$,
  $r_1,\dots,r_\ell$ are role names, $\mu:\{1,\dotsc,\ell\}
  \rightarrow R$ assigns to each set $S_i$ a variable from $R$, $\nu:R
  \rightarrow \mn{Ind}(\Amc)$ assigns to each variable from $R$ and
  individual name from \Amc, and the following conditions are
  satisfied:
  \begin{enumerate}
  \item $\mn{avar}(q) \subseteq R$ and $x=y \in q$ implies $\nu(x) = \nu(y)$;
  \item if $r(x,y) \in q$ with $x,y \in R$, then $r(\nu(x),\nu(y)) \in \Amc$;
    \item $q|_{S_i}$ is tree-shaped and can thus be seen as an
      \EL-concept $C_{q|S_i}$, for $1\leq i\leq \ell$;
    \item if $r(x,x') \in q$ then either (i) $x,x'$ belong to the same set 
    $R,S_1,\dotsc,S_\ell$, or (ii) $x\in R$ and, for some $i$, $r=r_i$
    and $x'$ root of $q|_{S_i}$.
  \end{enumerate}
\end{definition}
The following lemma illustrates the combined use and raison d'\^etre
of both fork rewritings and splittings. A proof is standard and
omitted, see for example \cite{Lutz-DL08}. It does rely on the
existence of \emph{forest models} for ABoxes and \EL-TBoxes,
that is, for every ABox \Amc and TBox $\Tmc$\!, there is a model \Imc
whose shape is that of \Amc with a directed (potentially infinite) tree
attached to each individual. 
\begin{lemma}
\label{lem:splitting}
  Let $Q=(\Tmc\!,\Sigma,q_0)$ be an OMQ from $(\EL,\text{CQ})$, \Amc a
  $\Sigma$-ABox, and $\abf \subseteq \mn{Ind}(\Amc)$. Then $\Amc
  \models Q(\abf)$ iff there exists a fork rewriting $q$ of $q_0$ and
  a splitting $\langle R,S_1,\dotsc,S_\ell,r_1,\dots,r_\ell,\mu,\nu \rangle$ of $q$
  w.r.t.\ \Amc and~\Tmc such that the following conditions are
  satisfied:
  \begin{enumerate}

  \item $\nu(\xbf)=\abf$, $\xbf$ the answer variables of $q_0$;
  \item if $A(x) \in q$ and $x \in R$, then $\Amc,\Tmc \models A(\nu(x))$;
  \item $\Amc,\Tmc \models \exists r_i.C_{q|_{S_i}}(\nu(\mu(i)))$ for
    $1 \leq i \leq \ell$.

  \end{enumerate}
\end{lemma}
\begin{example} \label{ex:splitting}
  To illustrate the described notions, consider the following CQ.
  \begin{align*}
    q_3(x) = \exists y_1, y_2, z\; &\mn{Person}(x) \; \wedge \\
    &\mn{hasDisease}(x,y_1) \wedge \mn{MelaminDeficiency}(y_1) \wedge 
    \mn{causedBy}(y_1, z)\;\wedge \\ 
    &\mn{hasDisease}(x,y_2) \wedge \mn{ImpairedVision}(y_2) \wedge 
    \mn{causedBy}(y_2, z)\;\wedge \\
    &\mn{GeneDefect}(z)
  \end{align*}
  It asks for persons suffering from two conditions connected with the
  same gene defect. Let the ABox \Amc consist only of the assertion
  $\mn{OCA1aPatient}(a)$. We extend the TBox $\Tmc_2$ from
  Example~\ref{ex:fonoyes}, as follows:
  \begin{align*}
    \Tmc_3 := \Tmc_2 \cup \{ \ \
    \mn{OCA1aPatient} &\sqsubseteq \mn{Person} \sqcap 
    \mn{hasDisease.OCA1aAlbinism} \\
    \mn{OCA1aAlbinism} &\sqsubseteq \mn{ImpairedVision} \sqcap 
    \mn{MelaninDeficiency} \\
    \mn{OCA1aAlbinism} &\sqsubseteq \exists \mn{causedBy.GeneDefect}
    \ \}
  \end{align*}
  Let  $Q = (\Tmc_3, \Sigma_\mn{full}, q_3(x))$. It can be verified
  that $\Amc \models Q(a)$. By Lemma~\ref{lem:splitting}, this is
  witnessed by a fork rewriting and a splitting $\Pi$. The
  fork rewriting is
  \begin{align*}
    q_3'(x) = \exists y_1, z\; &\mn{Person}(x) \wedge \\
    &\mn{hasDisease}(x,y_1) \wedge \mn{MelaminDeficiency}(y_1) \wedge 
    \mn{ImpairedVision}(y_1) \;\wedge \\
    &\mn{causedBy}(y_1, z) \wedge \mn{GeneDefect}(z)
  \end{align*}
  The splitting $\Pi=\langle R,S_1,r_1,\mu,\nu \rangle$ of $q_3'$ wrt.\
  \Amc and $\Tmc_3$ is defined by setting
  \begin{align*}
    R = \{ x \}, \; S_1 = \{ y_1, z \}, \; r_1 = \mn{hasDisease}, \; \mu(1) = 
    x, \; \nu = (x \mapsto a)
  \end{align*}
  It can be verified that the conditions given in Lemma~\ref{lem:splitting}
  are satisfied.
  \myEndEx
\end{example}

\section{Tree-quantified CQs}
\label{sect:tqCQs}

We reduce FO-rewritability in $(\EL, \text{tqCQ})$ to FO-rewritability
in $(\EL,\text{AQ})$ and, making only very mild assumptions on the
algorithm used for solving the latter problem, show that rewritings of
the OMQs produced in the reduction can be transformed in a
straightforward way into rewritings of the original OMQ. The mild
assumptions are that the algorithm produces a tUCQ-rewriting and that,
informally, when constructing the tCQs of the tUCQ-rewriting it never
introduces atoms `without a reason'---this will be made precise
later. 

Let $Q=(\Tmc\!,\Sigma,q_0)$ be from $(\EL, \text{tqCQ})$. We can
assume w.l.o.g.\ that $q_0$ contains only answer variables: every tCQ
in $q$ with root $x$ can be represented as an \EL-concept $C$ and we
can replace the tree with the atom $A_C(x)$ (unless it has only a
single node) and extend $\Tmc$ with $C \sqsubseteq A_C$ where $A_C$ is
a fresh concept name that is not included in $\Sigma$. Clearly, the
resulting OMQ is equivalent to the original
one. 
%
%
%

Let $Q$ be an OMQ from $(\EL,\text{tqCQ})$.  We show how to construct
an OMQ $Q'=(\Tmc',\Sigma',q'_0)$ from $(\EL,\text{AQ})$ with the
announced properties; in particular, $Q$ is FO-rewritable if and only
if $Q'$ is. Let $\mn{CN}(\Tmc)$ and $\mn{RN}(\Tmc)$ denote the set of
concept names and role names that occur in $\Tmc$\!, and let $\mn{sub}_L$
denote the set of concepts that occur on the left-hand side of a concept
inclusion in $\Tmc$\!, closed under subconcepts. Reserve a fresh concept
name $A^x$ for every $A \in \mn{CN}(\Tmc)$ and $x \in \mn{avar}(q_0)$,
and a fresh role name $r^x$ for every $r \in \mn{RN}(\Tmc)$ and $x \in
\mn{avar}(q_0)$. Set 
%
\begin{align*}
  \Sigma' = \Sigma \; &\cup \; \{ A^x \mid A \in \mn{CN}(\Tmc) \cap \Sigma
  \text{ and } x \in \mn{avar}(q_0) \} \; \\[0.5mm]
  &\cup \; \{ r^x \mid r \in \mn{RN}(\Tmc) \cap \Sigma \text{ and } x \in 
  \mn{avar}(q_0) \}.
\end{align*}
Additionally reserve a concept name $A_{\exists r.E}^x$ for every
concept $\exists r . E \in \mn{sub}_L(\Tmc)$ and every $x \in
\mn{avar}(q_0)$. Define
\begin{align*}
 \Tmc' := \Tmc \; &\cup\;  \{ C^x_L \sqsubseteq D^x_R \mid 
x\in\mn{var}(q_0) \text{ and } C\sqsubseteq D\in\Tmc \} \\[1mm]
%
&\cup\; \{ \exists r^x.C \sqsubseteq A_{\exists r.C}^x \mid 
x\in\mn{var}(q_0) 
\text{ and } \exists r.C\in\mn{sub}_L(\Tmc) \} \\[1mm]
%
& \cup\; \{ C^y_L \sqsubseteq A_{\exists r.C}^x \mid r(x,y)\in q_0
\text{ and }
\exists r.C \in \mn{sub}_L(\Tmc) \} \\[1mm]
&\cup \; \{\bigsqcap_{A(x)\in q_0} A^x \sqsubseteq N \} 
\end{align*}
where for a concept $C = A_1 \sqcap \dotsb \sqcap A_n \sqcap \exists r_1.E_1 
\sqcap \dotsb \sqcap \exists r_m.E_m$, the concepts $C^x_L$ and $C^x_R$ are 
given by 
\begin{align*}
C^x_L &= A^x_1 \sqcap \dotsb \sqcap A^x_n \sqcap A_{\exists r_1.E_1}^x \sqcap 
\dotsb \sqcap A_{\exists r_m.E_m}^x \\
 C^x_R &= A^x_1 \sqcap \dotsb \sqcap A^x_n \sqcap \exists r^x_1.E_1 \sqcap 
\dotsb \sqcap \exists r^x_m.E_m 
\end{align*}
Moreover, set $q'_0 := N(x)$. 

\begin{example}
  Consider the OMQ $Q = (\Tmc_1, \Sigma_\mn{full}, q(x,y))$ with
  $\Tmc_1$ as in Example~\ref{ex:omq}
and let $q(x,y)$ the following
  tqCQ:\footnote{We only
    use here that $\Tmc_1$ contains the concept $\exists
    \mn{hasDisease}.\mn{HereditaryDisease}$ on the left-hand side of a
    concept
    inclusion.}
  \begin{align*}
    q(x,y) = \exists z \; &\mn{GeneticRiskPatient}(x) \wedge 
    \mn{hasDisease}(x,y) \;\wedge \\ &\mn{Disease}(y) \wedge 
    \mn{hasDisease}(x,z) \wedge \mn{Albinism}(z)
  \end{align*}
  We first remove quantified variables: all atoms that contain the variable
  $z$ are replaced by $A_{\exists \mn{hasDisease}.\mn{Albinism}}(y)$,
  and the TBox is extended with the inclusion $\exists \mn{hasDisease}.\mn{Albinism}
  \sqsubseteq A_{\exists \mn{hasDisease.Albinism}}$. We then construct
  $\Tmc'_1$, which we give here only partially. The final concept
  inclusion in $\Tmc_1$ is
  \[ \mn{GeneticRiskPatient}^x \sqcap \mn{Disease}^y \sqcap 
  A_{\exists \mn{hasDisease.Albinism}}^x \sqsubseteq N,\] 
  representing the updated query without role atoms; for example, the
  concept name $\mn{Disease}^y$ stands for the atom $\mn{Disease}(y)$.
  Among others, $\Tmc_1'$ contains the further concept inclusions
  $$
  \begin{array}{rcl}
  \exists \mn{hasDisease}^x.\mn{HereditaryDisease} &\sqsubseteq& A^x_{\exists
    \mn{hasDisease.HereditaryDisease}} \\[1mm]
  \mn{HereditaryDisease}^y &\sqsubseteq& A^x_{\exists
    \mn{hasDisease.HereditaryDisease}} 
  \end{array}
  $$
  where, intuitively, the lower concept inclusion captures that case
  that the truth of the concept $\exists
  \mn{hasDisease}.\mn{HereditaryDisease}$ is witnessed at $y$ (the
  role atom $\mn{hasDisease}(x,y)$ from $q$ is only implicit here)
  while the upper concept inclusion deals with other witnesses.
%
  %
  \myEndEx 
\end{example}

Before proving that the constructed OMQ $Q'$ behaves in the desired
way, we give some preliminaries. It is known that, if an OMQ from
$(\EL,\text{AQ})$ has an FO-rewriting, then it has a tUCQ-rewriting,
see for example \cite{IJCAI13,IJCAI15}.
%
A tCQ $q$ is
\emph{conformant} if it satisfies the following properties:
\begin{enumerate}

\item if $A(x)$ is a concept atom, then either $A$ is of the form
  $B^y$ and $x$ is the answer variable or $A$ is not of this form
  and $x$ is a quantified variable;

\item if $r(x,y)$ is a role atom, then either $r$ is of the form
  $s^z$ and $x$ is the answer variable or $r$ is not of this form
  and $x$ is a quantified variable.

\end{enumerate}
A \emph{conformant tUCQ} is then defined in the expected way.  The
notion of conformance captures what we informally described as never
introducing atoms into the rewriting `without a reason'.  By the
following lemma, FO-rewritability of the OMQs constructed in our
reduction implies conformant tUCQ-rewritability, that is, there is
indeed no reason to introduce any of the atoms that are forbidden in
conformant rewritings.
\begin{restatable}{lemma}{LEMredOnecorrect}\label{lem:red1-correct} 
  Let $Q$ be from $(\EL, \text{tqCQ})$ and $Q'$ the OMQ constructed from $Q$ 
  as above. If $Q'$ is FO-rewritable, then it is rewritable into a conformant
  tUCQ.
\end{restatable}
When started on an OMQ produced by our reduction, the algorithms
presented in \cite{IJCAI15} and implemented in the Grind system
produce a conformant tUCQ-rewriting. Indeed, this can be expected of
any reasonable algorithm based on backwards chaining. Let $q'$ be a
conformant tUCQ-rewriting of $Q'$. The \emph{corresponding UCQ for
  $Q$} is the UCQ $q$ obtained by taking each CQ from $q'$, replacing
every atom $A^x(x_0)$ with $A(x)$ and every atom $r^x(x_0,y)$ with
$r(x,y)$, and adding all atoms $r(x,y)$ from $q_0$ such that both $x$
and $y$ are answer variables.  The answer variables in $q$ are those
of $q_0$. Observe that $q$ is a union of tqCQs.

\begin{restatable}{proposition}{PROPconformant}\label{prop:conformant}
  $Q$ is FO-rewritable iff $Q'$ is FO-rewritable. Moreover, if $q'$ is 
  a conformant tUCQ-rewriting of $Q'$ and $q$ the corresponding UCQ 
  for $Q$, then $q$ is a rewriting of~$Q$. 
\end{restatable}
The proof strategy is to establish the `moreover' part and to
additionally show how certain UCQ-rewritings of $Q$ can be converted
into UCQ-rewritings of $Q'$. More precisely, a CQ $q$ is a
\emph{derivative} of $q_0$ if it results from $q_0$ by exchanging
atoms $A(x)$ for \EL-concepts $C$, seen as tree-shaped CQs rooted in
$x$. We are going to prove the following lemma in
Section~\ref{sect:rCQs}.
\begin{lemma}
\label{lem:french}
If an OMQ $(\Tmc\!,\Sigma,q_0)$ from $(\EL,\text{tqCQ})$ is
FO-rewritable, then it has a UCQ-rewriting in which each CQ is a
derivative of $q_0$.
\end{lemma}
Let $q$ be a UCQ in which every CQ is a derivative of $q_0$. Then the
\emph{corresponding UCQ for $Q'$} is the UCQ $q'$ obtained by taking
each CQ from $q$, replacing every atom $A(x)$, $x$ answer variable,
with $A^x(x_0)$, every atom $r(x,y)$, $x$ answer variable and $y$
quantified variable, with $r^x(x_0,y)$, and deleting all atoms
$r(x_1,x_2)$, $x_1,x_2$ answer variables. The answer variable in $q'$
is $x_0$. Note that $q'$ is a tUCQ. To establish the ``only if''
direction of Proposition~\ref{prop:conformant}, we show that when $q$
is a UCQ-rewriting of $Q$ in which every CQ is a derivative of the
query $q_0$, then the corresponding UCQ for $Q'$ is a rewriting of
$Q'$.

\section{Rooted CQs}
\label{sect:rCQs}

We consider OMQs based on rCQs, a strict generalization of tqCQs. In
this case, we are not going to achieve a `black box' reduction, but
rely on a concrete algorithm for solving FO-rewritability in
$(\EL,\text{AQ})$. This algorithm is a straightforward and not
necessarily terminating backwards chaining algorithm or a
(potentially terminating) refinement thereof, as implemented in the
Grind system. We show how to combine the construction of (several)
OMQs from $(\EL,\text{AQ})$ with a modification of the assumed
algorithm to decide FO-rewritability in $(\EL,\text{rCQ})$ and to
construct actual
rewritings.  

We start with introducing the straightforward backwards chaining
algorithm mentioned above which we refer to as
$\mn{bc}_\text{AQ}$. Central to $\mn{bc}_\text{AQ}$ is a backwards 
chaining step based on concept inclusions 
in the TBox used in the OMQ. 
Let $C$ and $D$ be \EL-concepts, $E \sqsubseteq F$ a concept
inclusion, and $x \in \mn{var}(C)$ (where $C$ is viewed as a
tree-shaped CQ). Then $D$ is \emph{obtained from $C$ by applying $E
  \sqsubseteq F$ at $x$} if $D$ can be obtained from $C$ by
\begin{itemize}

\item removing $A(x)$ for all concept names $A$ with $\models F
  \sqsubseteq A$;

 \item removing $r(x,y)$ and the tree-shaped CQ $G$ rooted at $y$ when 
   $\models F \sqsubseteq \exists r . G$;


\item adding $A(x)$ for all concept names $A$ that occur in $E$ as a
  top-level conjunct (that is, that are not nested inside existential
  restrictions);

\item adding $\exists r . G$ as a CQ with root $x$, for each
  $\exists r.G$ that is a top-level conjunct of~$E$.

\end{itemize}
%
%
Let $C$ and $D$ be \EL-concepts. We write $D \prec C$ if $D$ can
be obtained from $C$ by removing an existential restriction (not
necessarily on top level, and potentially resulting in $D=\top$ when
$C$ is of the form $\exists r . E$). We use $\prec^\ast$ to denote the
reflexive and transitive closure of~$\prec$ and say that \emph{$D$ is
  $\prec$-minimal with $\Tmc \models D \sqsubseteq A_0$} if $\Tmc
\models D \sqsubseteq A_0$ and there is no $D' \prec D$ with $\Tmc
\models D' \sqsubseteq A_0$.

Now we are in the position to describe algorithm
$\mn{bc}_\text{AQ}$. It maintains a set $M$ of \EL-concepts that
represent tCQs. Let $Q = (\Tmc\!, \Sigma, A_0)$ be from $(\EL,
\text{AQ})$. Starting from the set $M = \{ A_0 \}$, it exhaustively
performs the following steps:
\begin{enumerate}
\item find $C \in M$, $x \in \mn{var}(C)$, a concept
  inclusion 
  $E \sqsubseteq F \in \Tmc$, and $D$, such that $D$ is obtained from
  $C$ by applying $E \sqsubseteq F$ at $x$;
  \item find $D' \prec^\ast D$ that is $\prec$-minimal with $\Tmc 
  \models D' \sqsubseteq A_0$, and add $D'$ to $M$.
\end{enumerate}
Application of these steps might not terminate. 
We use
$\mn{bc}_\text{AQ}(Q)$ to denote the potentially infinitary UCQ
$\bigvee M|_\Sigma$ where $M$ is the set obtained in the limit and
$q|_\Sigma$ denotes the restriction of the UCQ $q$ to those disjuncts
that only use symbols from
$\Sigma$.  Note that, in Point~2, it is possible to
 find the desired $D'$ in polynomial time since the subsumption `$\Tmc 
  \models D' \sqsubseteq A_0$' can be decided in polynomial time.
The following is standard to prove, see 
\cite{TheFrench,IJCAI15} and Lemma~\ref{lem:soundcompl} below for
similar results.
\begin{lemma}
\label{lem:bcAQ}
Let $Q$ be an OMQ from $(\EL,\text{AQ})$. If $\mn{bc}_\text{AQ}(Q)$ is
finite, then it is a UCQ-rewriting of $Q$.  Otherwise, $Q$ is
not FO-rewritable.  
\end{lemma}
\begin{example}
  Consider the TBox
  $$
    \Tmc = \{\mn{Person} \sqcap \exists \mn{hasParent.GeneticRiskPatient}     
    \sqsubseteq \mn{GeneticRiskPatient} \}
    $$
  and let $Q = (\Tmc\!, \Sigma,\mn{GeneticRiskPatient}(x))$ with $\Sigma =
  \{\mn{Person},\mn{GeneticRiskPatient}\}$. Note that the role name
  $\mn{hasParent}$ does not occur in $\Sigma$. Even though the set $M$
  generated by $\mn{bc}_\text{AQ}$ (in the limit of its non-terminating
  run) is infinite, $\mn{bc}_\text{AQ}(Q) =
  \mn{GeneticRiskPatient}(x)$ is finite and a UCQ-rewriting of $Q$.
  \myEndEx
\end{example}

The algorithm for deciding FO-rewritability in $(\EL,\text{AQ})$
presented in \cite{IJCAI15} and underlying the Grind system can be
seen as a refinement of $\mn{bc}_\text{AQ}$. Indeed, that algorithm
always terminates and returns $\bigvee M|_\Sigma$ if that UCQ is
finite and reports non-FO-rewritability otherwise. Moreover, the
UCQ-rewriting is represented in a decomposed way and output as a
non-recursive Datalog program for efficiency and succinctness.  For
our purposes, the only important aspect is that, when started on an
FO-rewritable OMQ, it computes (a non-recursive Datalog program that
is equivalent to) the UCQ-rewriting $\bigvee M|_\Sigma$.

\smallskip

We next introduce a generalized version $\mn{bc}^+_{\text{AQ}}$ of
$\mn{bc}_{\text{AQ}}$ that takes as input an OMQ
$Q=(\Tmc\!,\Sigma,A_0)$ from $(\EL, \text{AQ})$ and an
additional \EL-TBox $\Tmc^{\mn{min}}$, such that termination and
output of $\mn{bc}^+_{\text{AQ}}$ agrees with that of
$\mn{bc}_{\text{AQ}}$ when the input satisfies $\Tmc^{\mn{min}}=\Tmc$.
Starting from $M = \{ A_0 \}$, algorithm $\mn{bc}^+_{\text{AQ}}$ 
exhaustively performs the following steps:
\begin{enumerate}
\item find $C \in M$, $x \in \mn{var}(C)$, a concept
  inclusion 
  $E \sqsubseteq F \in \Tmc$, and $D$, such that $D$ is obtained from
  $C$ by applying $E \sqsubseteq F$ at $x$;
  \item find $D' \prec^\ast D$ that is $\prec$-minimal with 
  $\Tmc^{\mn{min}} \models D' \sqsubseteq A_0$, and add $D'$ to $M$.
\end{enumerate}
We use $\mn{bc}^+_\text{AQ}(Q, \Tmc^{\mn{min}})$ to denote the potentially 
infinitary UCQ $\bigvee M|_\Sigma$, $M$ obtained in the limit. Note that
$\mn{bc}^+_{\text{AQ}}$ uses the TBox \Tmc for backwards chaining and
$\Tmc^{\mn{min}}$ for minimization while $\mn{bc}_{\text{AQ}}$ uses
\Tmc for both purposes.  The refined version of $\mn{bc}_{\text{AQ}}$
implemented in the Grind system can easily be adapted  to behave
like a terminating version of $\mn{bc}^+_{\text{AQ}}$.

\medskip

Our aim is to convert an OMQ $Q=(\Tmc\!,\Sigma,q_0)$ from
$(\EL,\text{rCQ})$ into a set of pairs $(Q',\Tmc^{\mn{min}})$ with
$Q'$ an OMQ from $(\EL,\text{AQ})$ and $\Tmc^{\mn{min}}$ an \EL-TBox
such that $Q$ is FO-rewritable iff
$\mn{bc}^+_{\text{AQ}}(Q',\Tmc^{\mn{min}})$ terminates for all pairs
$(Q',\Tmc^{\mn{min}})$ and, moreover, if this is the case, then the
resulting UCQ-rewritings can straightforwardly be converted into a
rewriting of $Q$.

Let $Q=(\Tmc\!,\Sigma,q_0)$. We construct one pair
$(Q_{q_r},\Tmc_{q_r}^{\mn{min}})$ for each fork rewriting $q_r$ of
$q_0$. 
We use $\mn{core}(q_r)$ to denote the minimal set $V$ of variables
that contains all answer variables in $q_r$ and such that after
removing all atoms $r(x,y)$ with $x,y \in V$, we obtain a disjoint
union of tree-shaped CQs. We call these CQs the \emph{trees in $q_r$}.
%
Intuitively, we separate the tree-shaped parts of $q_r$ from the
cyclic part, the latter identified by $\mn{core}(q_r)$. This is
similar to the definition of tqCQs where, however, cycles cannot
involve any quantified variables.  In a forest model of an ABox and a
TBox as mentioned before Lemma~\ref{lem:splitting}, the variables in
$\mn{core}(q_r)$ must be mapped to the ABox part of the model (rather
than to the trees attached to it).
%
Now $(Q_{q_r},\Tmc_{{q_r}}^{\mn{min}})$ is defined by setting
$Q_{q_r}=(\Tmc_{q_r},\Sigma_{q_r},N(x))$ and
\begin{align*}
\Tmc_{q_r} = \Tmc \; &\cup\; \{ C^x_R \sqsubseteq D^x_R \mid 
x\in\text{core}({q_r}), C\sqsubseteq D\in\Tmc \} \\[1mm]
&\cup \; \{ \bigsqcap_{C(x) \text{ a tree in } q_r} C_R^x \sqsubseteq N \}
\end{align*}
where $C^x_R$ is defined as in Section~\ref{sect:tqCQs},
%
and $\Sigma_{q_r}$ is the extension of $\Sigma$ with all concept names
$A^x$ and role names $r^x$ used in $\Tmc_{q_r}$ such that $A,r \in
\Sigma$. 

It remains to define $\Tmc_{{q_r}}^{\mn{min}}$, which is
$\Tmc_{q_r}$ extended with one concept inclusion for each fork
rewriting $q$ of $q_0$ and each splitting $\Pi = \langle
R,S_1,\dotsc,S_\ell,r_1,\dots,r_\ell,\mu,\nu \rangle$ of $q$ w.r.t.\
$\Amc_{q_r}$, as follows. For each $x \in \mn{avar}({q_r})$, the
equality atoms in ${q_r}$ give rise to an equivalence class
$[x]_{q_r}$ of answer variables, defined in the expected
way. 
We only consider the splitting $\Pi$ of $q$ 
if it preserves answer variables modulo equality, that is, if $x \in
\mn{avar}(q)$, then there is a $y \in [x]_{q_r}$ such that $\nu(x)=y$.
We then add the inclusion
%
%
$$
\Big( \bigsqcap_{\substack{A(x)\in q \\ \text{ with }
      x\in R}} A^{\nu(x)} \Big) \; \sqcap \;
  \Big(\bigsqcap_{1\leq i\leq\ell} \exists
  r_i^{\nu(\mu(i))}.C_{q|_{S_i}} \Big) \;\; \sqsubseteq \;\; N 
$$
It can be shown that, summing up over all fork rewritings and
splittings, only polynomially many concepts $\exists
r_i^{\nu(\mu(i))}.  C_{q|_{S_i}}$ are introduced (this is similar to
the proof of Lemma~6 in
\cite{Lutz-DL08}). Note that we do not introduce fresh concept
names of the form $A^x_{\exists r . C}$ as in
Section~\ref{sect:tqCQs}. This is not necessary here because of the use
of fork rewritings and splittings in $\Tmc^{\mn{min}}$.


\begin{example} \label{ex:rcq} Consider query $q_3$ from
  Example~\ref{ex:splitting} and TBox $\Tmc_1$ from
  Example~\ref{ex:omq}. Constructing $\Tmc_{q_3}$ (thus considering
  $q_3$ as a fork rewriting of itself) would add concept inclusions
  like
  \[ \mn{Person}^x \sqcap \exists \mn{hasDisease}^x.\mn{HereditaryDisease} 
  \sqsubseteq 
  \mn{GeneticRiskPatient}^x
  \]
  The final concept inclusion added is the following, listing concepts needed
  at $x, y_1, y_2,$ and $z$ that result in a match of $q_3$:
  \[ \mn{Person}^x \sqcap \mn{MelaminDeficiency}^{y_1} \sqcap 
  \mn{ImpairedVision}^{y_2} \sqcap \mn{GeneDefect}^z \sqsubseteq N
  \]
  When building the TBox $\Tmc^\mn{min}_{q_3}$, it is necessary to
  look for matches of $q_3$ by a splitting $\Pi$ of a fork rewriting
  of $q_3$ w.r.t.\ $\Amc_{q_3}$ and $\Tmc_1$. We consider here the
  splitting $\Pi=\langle R,S_1,r_1,\mu,\nu \rangle$ of the fork
  rewriting $q_3'$ of $q_3$ given in Example~\ref{ex:splitting},
  defined by setting
  \begin{align*}
    R = \{ x \}, \; S_1 = \{ y_1, z \}, \; r_1 = \mn{hasDisease}, \; \mu(1) = 
    x, \; \nu = (x \mapsto x)
  \end{align*}
  For $\Pi$, the following concept inclusion   is added to $\Tmc^\mn{min}_{q_3}$:
  \begin{align*}
    \mn{Person}^x \sqcap \exists \mn{hasDisease}^{x}.\big(\mn{MelaminDeficiency} 
    \sqcap \mn{ImpairedVision} \; \sqcap \quad & \\ 
    \mn{causedBy.GeneDefect}\big) &\sqsubseteq N  \ghost{~~~~\myEndEx}
  \end{align*}
\end{example}
It can be seen that when
$\mn{bc}^+_{\text{AQ}}(Q_{q_r},\Tmc^{\mn{min}}_{q_r})$ is finite, then
it is a conformant tUCQ in the sense of
Section~\ref{sect:tqCQs}. Thus, we can also define a
\emph{corresponding UCQ $q$ for $Q$} as in that section, that is, $q$
is obtained by taking each CQ from $q'$, replacing every atom
$A^x(x_0)$ with $A(x)$ and every atom $r^x(x_0,y)$ with $r(x,y)$, and
adding all atoms $r(x,y)$ from ${q_r}$ such that $x,y \in
\mn{core}({q_r})$.  The answer variables in $q$ are those of
$q_0$. 
%
\begin{restatable}{proposition}{LEMredTwocorrect}\label{lem:red2-correct}
  Let $Q=(\Tmc\!,\Sigma,q_0)$ be an OMQ from $(\EL,\text{rCQ})$. If
  $\mn{bc}^+_{\text{AQ}}(Q_{q_r},$ $\Tmc_{q_r}^{\mn{min}})$ is finite
  for all fork rewritings ${q_r}$ of $q_0$, then $\bigvee_{q_r}
  \widehat q_{q_r}$ is a UCQ-rewriting of $Q$, where $\widehat
  q_{q_r}$ is the UCQ for $Q$ that corresponds to
  $\mn{bc}^+_{\text{AQ}}(Q_{q_r},\Tmc_{q_r}^{\mn{min}})$. Otherwise,
  $Q$ is not FO-rewritable.
\end{restatable}
To prove Proposition~\ref{lem:red2-correct}, we introduce a backwards
chaining algorithm $\mn{bc}_{\text{rCQ}}$ for computing UCQ-rewritings
of OMQs from $(\EL,\text{rCQ})$ that we refer to as
$\mn{bc}_{\text{rCQ}}$. In a sense, $\mn{bc}_{\text{rCQ}}$ is the
natural generalization of $\mn{bc}_{\text{AQ}}$ to rCQs. We then show
a correspondence between the run of $\mn{bc}_{\text{rCQ}}$ on the 
input OMQ $Q$ from $(\EL,\text{rCQ})$ and the runs of
$\mn{bc}^+_{\text{AQ}}$ on the constructed inputs of the form 
$(Q_{q_r},\Tmc_{q_r}^{\mn{min}})$.

\smallskip

On the way, we also provide the missing proof for
Lemma~\ref{lem:french}, which in fact is a consequence of the
correctness of $\mn{bc}_{\text{rCQ}}$ (states as
Lemma~\ref{lem:soundcompl} in the appendix) and the observation that,
when $Q=(\Tmc\!,\Sigma,q_0)$ is from $(\EL,\text{tqCQ})$, then
$\mn{bc}_{\text{rCQ}}(Q)$ contains only derivatives of $q_0$. The
latter is due to the definition of the $\mn{bc}_{\text{rCQ}}$
algorithm, which starts with a set of minimized fork rewritings of
$q_0$, and the fact that the only fork rewriting of a tqCQ is the
query itself.

\medskip

There are two exponential blowups in the presented approach. First,
the number of fork rewritings of $q_0$ might be exponential in the
size of $q_0$. We expect this not to be a problem in practice since
the number of fork rewritings of realistic queries should be fairly
small. And second, the number of splittings can be exponential and
thus the same is true for the size of each $\Tmc_{q_r}^{\mn{min}}$. We expect
that also this blowup will be moderate in practice. Moreover, in an
optimized implementation one would not represent $\Tmc_{q_r}^{\mn{min}}$ as
a TBox, but rather check the existence of fork rewritings and
splittings that give rise to concept inclusions in $\Tmc_{q_r}^{\mn{min}}$
in a more direct way. This involves checking whether
concepts of the form $\exists r_i^{\nu(\mu(i))}.  C_{q'|_{S_i}}$ are
derived, and the fact that there are only polynomially many different
such concepts should thus be very relevant regarding performance.

\section{Experiments}

We have extended the \emph{Grind} system \cite{IJCAI15} to support
OMQs from $(\EL,\text{tqCQ})$ and $(\EL,\text{rCQ})$ instead of only
from $(\EL,\text{AQ})$, 
and conducted experiments with real-world ontologies and hand-crafted conjunctive 
queries.
%
%
The system can be downloaded from 
\url{http://www.cs.uni-bremen.de/~hansen/grind}, together with the 
ontologies and 
queries, and is released under GPL.  
It outputs rewritings in the form of non-recursive Datalog
queries.  We have implemented the following optimization: given
$Q=(\Tmc\!,\Sigma,q_0)$, first compute all fork rewritings of $q_0$,
rewrite away all variables outside of the core (in the same way in which
tree parts of the query are removed in Section~\ref{sect:tqCQs}) to
obtain a new OMQ $(\Tmc',\Sigma,q'_0)$, and then test for each atom
$A(x) \in q'_0$ whether $(\Tmc',\Sigma,A(x))$ is FO-rewritable. It can
be shown that, if this is the case, then $Q$ is FO-rewritable, and it
is also possible to transfer the actual
rewritings. If this check fails, we go through the full construction 
described in the paper. 

Experiments were carried out on a Linux (3.2.0) machine with a 3.5~GHz
quad-core processor and 8~GB of RAM. For the experiments, we use (the
\EL part of) the ontologies ENVO, FBbi, SO, MOHSE, and not-galen.
The first three ontologies are from the biology domain, and are available 
through Bioportal\footnote{\url{https://bioportal.bioontology.org}}. MOHSE and 
not-galen are different versions of the GALEN 
ontology\footnote{\url{http://www.opengalen.org/}}, which describes medical 
terms.
Some statistics is given in Table~\ref{tab:tboxes}, namely the number
of concept inclusions (CI), concept names (CN), and role names (RN)
in each ontology. For each ontology, we hand-crafted 10 conjunctive queries
(three tqCQs and seven rCQs), varying in size from 2 to 5 variables
and showing several different topologies (see Fig.~\ref{fig:queries}
for a sample).

\begin{table}[t]
\centering 
\begin{tabular}
{|@{\;}l@{\;}|@{\;}r@{\;}|@{\;}r@{\;}|@{\;}r@{\;}|@{\;}r@{\;}|@{\;}r@{\;}|@{\;}r@{\;}|@{\;}r@{\;}|@{\;}r@{\;}|}
  \hline
  TBox      &   CI &   CN &  RN & Min CQ & Avg CQ & Max CQ &   Avg AQ &
  Aborts 
  \TBstrut \\
  \hline
  ENVO      & 1942 & 1558 &   7 &   0.2s &   1.5s &     7s &       1s & 0 
  \Tstrut\\
  FBbi      &  567 &  517 &   1 &  0.05s &   0.5s &     3s &     0.3s & 0
  \Tstrut\\
  MOHSE     & 3665 & 2203 &  71 &     2s &    10s &    40s &       6s & 0
  \Tstrut\\ 
  not-galen & 4636 & 2748 & 159 &     6s &     9s &    28s &      25s & 2
  \Tstrut\\
  SO        & 3160 & 2095 &  12 &     1s &    19s &  2m23s &       4s & 1
  \Tstrut\\
  \hline
\end{tabular}
\vspace{3mm}
\caption{TBox information and results of experiments
}
\label{tab:tboxes}
 \vspace{-3mm}
\end{table}

\begin{figure}[t]
  \begin{boxedminipage}{\linewidth}
  { \footnotesize
  \begin{align*}
    ~\\[-8mm]
    q_1(x,y) = &\; \mn{Patient}(x) \wedge \mn{shows}(x,y) \wedge    
      \mn{Endocarditis}(y) \\
    q_2(w,x,y,z) = &\; \mn{Doctor}(w) \wedge \mn{hasPersonPerforming}(x,w) 
      \wedge \mn{Surgery}(x) \; \wedge \\
      &\; \mn{actsOn}(x,y) \wedge \mn{Tissue}(y) \wedge \mn{actsOn}(x,z) \; 
      \wedge \\
      &\; \mn{InternalOrgan}(z) \wedge \mn{hasAlphaConnection}(y,z) \\
    q_7(x) = &\; \exists y,z \; \mn{Protein}(x) \wedge \mn{contains}(x,y)   
      \wedge \mn{Tetracycline}(y) \; \wedge \\ 
      &\; \mn{InternalOrgan}(z) \wedge \mn{isActedOnSpecificallyBy}(z,y) \\
    q_8(x) = &\; \exists v,w,y,z \; \mn{Sulphonamide}(v) \wedge   
      \mn{serves}(v,w) \wedge \mn{TumorMarkerRole}(w) \; \wedge \\
      &\; \mn{NamedEnzyme}(x) \wedge \mn{serves}(x,w) \wedge \mn{actsOn}(x,z) 
      \wedge \mn{Liver}(z) \; \wedge \\
      &\; \mn{TeichoicAcid}(y) \wedge 
      \mn{actsOn}(y,z) \\
    q_{10}(x) = &\; \exists y,z \; \mn{BodyStructure}(x) \wedge 
    \mn{isBetaConnectionOf}(x,y)   
      \wedge \mn{Brain}(y) \; \wedge \\ 
      &\; \mn{IntrinsicallyNormalBodyStructure}(z) \wedge 
      \mn{isBetaConnectionOf}(z,y) \\[-5mm]
  \end{align*}
  }%
  \end{boxedminipage}
  \caption{Examplary queries used for experiments with TBox not-galen.}
  \label{fig:queries}
\end{figure}

The runtimes are reported in Table~\ref{tab:tboxes}. Only three
queries did not terminate in 30 minutes or exhausted the memory. For
the successful ones, we list fastest (Min CQ), slowest (Max CQ), and
average runtime (Avg CQ). For comparison, the Avg AQ column lists the
time needed to compute FO-rewritings for all queries
$(\Tmc\!,\Sigma,A(x))$ with $A(x)$ an atom in $q_0$. This check is of
course incomplete for FO-rewritability of $Q$, but can be viewed as a
lower bound. 
A detailed picture of individual runtimes is given in
Figure~\ref{fig:querytimes}.

In summary, we believe that the outcome of our experiments is
promising. While runtimes are higher than in the AQ case, they are
still rather small given that we are dealing with an intricate static
analysis task and that many parts of our system have not been
seriously optimized. The queries with long runtimes or timeouts
contain AQs that are not FO-rewritable which forces the decomposed
algorithm implemented in Grind to enter a more expensive processing
phase.

\begin{figure}[t]
\centering
\begin{tikzpicture}
  \pgfplotsset{compat = 1.3} 
  \begin{semilogyaxis}
  [ 
    ybar, 
    ymin=0.04, 
    log origin=infty, 
    log ticks with fixed point, 
    width=\textwidth,
     xtick style={draw=none},
    height=55mm, 
    bar width=1.3mm, 
    ylabel={time in s}, 
    ylabel shift = -1mm,
    legend pos=north west,
    legend style={
    },
    symbolic x coords={
      envo01, envo02, envo03, envo04, envo05, 
      envo06, envo07, envo08, envo09, envo10,
      fbbi01, fbbi02, fbbi03, fbbi04, fbbi05, 
      fbbi06, fbbi07, fbbi08, fbbi09, fbbi10, 
      mohse01, mohse02, mohse03, mohse04, mohse05, 
      mohse06, mohse07, mohse08, mohse09, mohse10, 
      ng01, ng02, ng03, ng04, ng05, 
      ng06, ng07, ng08, ng09, ng10,
      so01, so02, so03, so04, so05, so06, so07, so08, so09, so10
    }, 
    xtick={
      envo01, envo02, envo03, envo04, envo05, 
      envo06, envo07, envo08, envo09, envo10,
      fbbi01, fbbi02, fbbi03, fbbi04, fbbi05, 
      fbbi06, fbbi07, fbbi08, fbbi09, fbbi10, 
      mohse01, mohse02, mohse03, mohse04, mohse05, 
      mohse06, mohse07, mohse08, mohse09, mohse10, 
      ng01, ng02, ng03, ng04, ng05, 
      ng06, ng07, ng08, ng09, ng10,
      so01, so02, so03, so04, so05, so06, so07, so08, so09, so10
    },
    xticklabels={
      ~,~,~,ENVO,~,~,~,~,~,~,
      ~,~,~,FBbi,~,~,~,~,~,~,
      ~,~,~,~,MOHSE,~,~,~,~,~,
      ~,~,~,~,~,not-galen,~,~,~,~,
      ~,~,~,~,~,~,SO,~,~,~,
    },
    ] 

    \addplot +[black,fill=white!80!blue, area legend] coordinates {
(envo01, 0.825)
(envo02, 0.383)
(envo03, 0.288)
(envo04, 0.253)
(envo05, 0.215)
(envo06, 0.196)
(envo07, 0.607)
(envo08, 0.273)
(envo09, 6.535)
(envo10, 5.627)
    };

    \addplot +[black,fill=white!80!blue, area legend] coordinates {
(fbbi01, 0.058)
(fbbi02, 0.072)
(fbbi03, 0.079)
(fbbi04, 0.095)
(fbbi05, 0.074)
(fbbi06, 0.054)
(fbbi07, 0.195)
(fbbi08, 0.964)
(fbbi09, 0.743)
(fbbi10, 2.905)
    };

    \addplot +[black,fill=white!80!blue, area legend] coordinates {
(mohse01, 1.946)
(mohse02, 2.175)
(mohse03, 1.773)
(mohse04, 2.17)
(mohse05, 2.248)
(mohse06, 2.227)
(mohse07, 2.186)
(mohse08, 7.828)
(mohse09, 36.727)
(mohse10, 39.041)
    };
    
    \addplot +[black,fill=white!80!blue, area legend] coordinates {
(ng01, 6.394)
(ng02, 5.9)
(ng03, 6.587)
(ng04, 6.852)
(ng05, 6.447)
(ng06, 6.261)
(ng07, 6.856)
(ng08, 28.442)
    };

    \addplot +[black,fill=white!80!blue, area legend] coordinates {
(so01, 1.044)
(so02, 2.57)
(so03, 0.883)
(so04, 16.87)
(so05, 1.005)
(so06, 1.323)
(so07, 2.053)
(so08, 4.173)
(so09, 143.816)
    };

  \addplot[black,sharp plot,update limits=false] coordinates {(envo01,5) 
  (so09,5)} node[above] at (axis cs:so03,5) {5 s};

  \addplot[black,sharp plot,update limits=false] coordinates {(envo01,45) 
  (so09,45)} node[above] at (axis cs:so03,45) {45 s};

  \legend{Total runtime per query} 
  \end{semilogyaxis}
\end{tikzpicture}
\vspace{-2mm}
\caption{Runtimes for individual OMQs, showing only non-aborting runs.}
\label{fig:querytimes}
\end{figure}
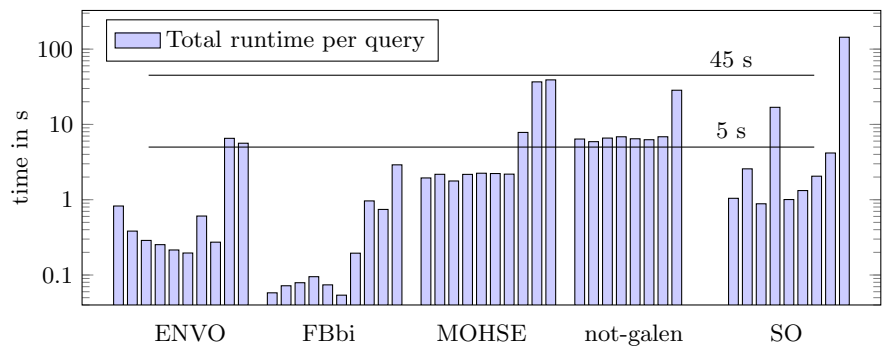

\section{Conclusion}

We remark that our approach can also be used to compute FO-rewritings
of OMQs from $(\EL,\text{CQ})$ even if the CQs are not rooted, as long
as they are not Boolean (that is, as long as they contain at least one
answer variable) and an algorithm for query containment in
$(\EL,\text{CQ})$ is also available. This follows from (a minor
variation of) an observation from \cite{IJCAI16}: FO-rewritability of
non-Boolean OMQs from $(\EL,\text{CQ})$ can be polynomially reduced to
a combination of containment in $(\EL,\text{CQ})$ and FO-rewritability
in $(\EL,\text{rCQ})$. As future work, it would be interesting to
extend our approach to UCQs, to the extension of \EL with role
hierarchies and domain and range restrictions, or even to the
extension \ELI of \EL with inverse roles.





\medskip
\noindent {\bf Acknowledgements.} We acknowledge support by
ERC grant 647289 `CODA'.


\bibliographystyle{splncs03}


\begin{thebibliography}{10}
\providecommand{\url}[1]{\texttt{#1}}
\providecommand{\urlprefix}{URL }

\bibitem{DLLiteJAIR09}
Artale, A., Calvanese, D., Kontchakov, R., Zakharyaschev, M.: The {DL-Lite}
  family and relations. J.\ Artif.\ Intell.\ Res.\ 36, pp.\ 1--69 (2009)

\bibitem{DL-Textbook}
Baader, F., Horrocks, I., Lutz, C., Sattler, U.: An introduction to description
  logics. Cambride University Press (2017)

\bibitem{PierisBarcelo17}
Barcel{\'{o}}, P., Berger, G., Pieris, A.: Containment for rule-based 
  ontology-mediated queries (19 April 2017), available at 
  \url{https://arxiv.org/abs/1703.07994} [cs.DB]

\bibitem{OurTODS14}
Bienvenu, M., ten Cate, B., Lutz, C., Wolter, F.: Ontology-based data access: A
  study through disjunctive datalog, {CSP}, and {MMSNP}. J.\ ACM Trans.\
  Database Syst.\ 39(4), pp.\ 33:1--33:44 (2014)

\bibitem{IJCAI16}
Bienvenu, M., Hansen, P., Lutz, C., Wolter, F.: First order-rewritability and
  containment of conjunctive queries in {Horn} description logics. In:
  Proc.\ of IJCAI, pp. 965--971 (2016)

\bibitem{IJCAI13}
Bienvenu, M., Lutz, C., Wolter, F.: First order-rewritability of atomic queries
  in {Horn} description logics. In: Proc.\ of IJCAI, pp.\ 754--760 (2013)

\bibitem{DLLiteJAR07}
Calvanese, D., {De Giacomo}, G., Lembo, D., Lenzerini, M., Rosati, R.:
  Tractable reasoning and efficient query answering in description logics: The
  {DL-Lite} family. J. Autom. Reasoning  39(3), pp.\ 385--429 (2007)

\bibitem{Backchase99}
Deutsch, A., Popa, L., Tannen, V.: Physical data independence, constraints, and
  optimization with universal plans. In: Proc.\ of VLDB, pp. 459--470 (1999)

\bibitem{EiterOSTX12}
Eiter, T., Ortiz, M., Simkus, M., Tran, T., Xiao, G.: Query rewriting for
  {Horn-SHIQ} plus rules. In: Proc.\ of AAAI (2012)

\bibitem{OurICDT17}
Feier, C., Lutz, C., Kuusisto, A.: Rewritability in monadic disjunctive
  datalog, {MMSNP}, and expressive description logics. In: Proc.\ of ICDT 
  (2017)

\bibitem{GottlobKKPSZ14}
Gottlob, G., Kikot, S., Kontchakov, R., Podolskii, V.V., Schwentick, T.,
  Zakharyaschev, M.: The price of query rewriting in ontology-based data
  access. J.\ Artif.\ Intell.\ 213, pp.\ 42--59 (2014)

\bibitem{IJCAI15}
Hansen, P., Lutz, C., Seylan, I., Wolter, F.: Efficient query rewriting in the
  description logic {EL} and beyond. In: Proc.\ of IJCAI, pp.\ 3034--3040 
  (2015)

\bibitem{KikotKZ12}
Kikot, S., Kontchakov, R., Zakharyaschev, M.: Conjunctive query answering with
  {OWL} 2 {QL}. In: Proc.\ of KR (2012)

\bibitem{KonigLM15}
K{\"{o}}nig, M., Lecl{\`{e}}re, M., Mugnier, M.: Query rewriting for
  existential rules with compiled preorder. In: Proc.\ of IJCAI, pp.\
  3106--3112 (2015)

\bibitem{TheFrench}
K{\"{o}}nig, M., Lecl{\`{e}}re, M., Mugnier, M., Thomazo, M.: Sound, complete
  and minimal {UCQ}-rewriting for existential rules. Semantic Web 6(5), pp.\
  451--475 (2015)

\bibitem{ijcar08}
Lutz, C.: The complexity of conjunctive query answering in expressive
  description logics. In: Proc.\ of IJCAR, pp. 179--193 (2008)

\bibitem{Lutz-DL08}
Lutz, C.: Two upper bounds for conjunctive query answering in {SHIQ}. In:
  Proc.\ of DL (2008)

\bibitem{LutzTomanWolterIJCAI09}
Lutz, C., Toman, D., Wolter, F.: Conjunctive query answering in the 
  description logic {EL} using a relational database system. In: Proc.\ of 
  IJCAI, pp.\ 2070--2075 (2009)

\bibitem{lutz-2012}
Lutz, C., Wolter, F.: Non-uniform data complexity of query answering in
  description logics. In: Proc.\ of KR (2012)

\bibitem{owlprofiles}
Motik, B., Cuenca Grau, B., Horrocks, I., Wu, Z., Fokoue, A., Lutz, C.: OWL~2 
  Web Ontology Language: Profiles. W3C recommendation (11 December 2012), 
  available at \url{http://www.w3.org/TR/owl2-profiles/}

\bibitem{Perez-UrbinaMH10}
P{\'{e}}rez{-}Urbina, H., Motik, B., Horrocks, I.: Tractable query answering
  and rewriting under description logic constraints. J. Applied Logic 8(2),
  pp.\ 186--209 (2010)

\bibitem{Rodriguez-MuroC12}
Rodriguez{-}Muro, M., Calvanese, D.: High performance query answering over
  DL-Lite ontologies. In: Proc.\ of KR (2012)

\bibitem{RosatiA10}
Rosati, R., Almatelli, A.: Improving query answering over DL-Lite ontologies.
  In: Proc.\ of KR (2010)

\bibitem{StefanoniEtAl}
Stefanoni, G., Motik, B.: Answering conjunctive queries over {EL} knowledge 
  bases with transitive and reflexive roles. In: Proc.\ of AAAI, pp.\ 
  1611--1617 (2015)

\bibitem{StefanoniMH12}
Stefanoni, G., Motik, B., Horrocks, I.: Small datalog query rewritings for
  {EL}. In: Proc.\ of DL (2012)

\bibitem{TrivelaSCS15}
Trivela, D., Stoilos, G., Chortaras, A., Stamou, G.B.: Optimising
  resolution-based rewriting algorithms for {OWL} ontologies. J.\ Web Sem.\ 33,
  pp.\ 30--49 (2015)

\end{thebibliography}


\clearpage
\appendix


{\noindent\LARGE\textbf{Appendix}}

\section{Proofs for Section~\ref{sect:tqCQs}}

\LEMredOnecorrect*
\begin{proof}
%
  (sketch) Assume that $Q'=(\Tmc',\Sigma',q_0')$ is FO-rewritable. Then there 
  is a tUCQ-re\-wri\-ting $\varphi$ \cite{IJCAI13,IJCAI15}. For all tCQs $q$ in
  $\varphi$, it holds that $q \subseteq_{\Tmc'} q_0'$. We can assume w.l.o.g.\
  that every CQ $q$ in $\varphi$ is $\subseteq$-minimal with this
  property. We argue that $\varphi$ must be conformant. Let $q$ be a CQ in
  $\varphi$.  First assume that $q$ contains a concept atom $A(x_0)$ (or
  role atom $r(x_0,y)$) where $x_0$ is the answer variable and $A$ not
  of the form $B^y$ (or $r$ not of the form $s^z$). Let $q^-$ be $q$
  without this atom. The fact that $q \subseteq_{\Tmc'} q_0'$ is equivalent
  to $a_{x_0}$ being an answer to query $Q'$ on the ABox $\Amc_q$,
  which is $q$ seen as an ABox with root $a_{x_0}$. Entailment of AQs
  under a TBox can be characterized by derivation trees, see
  e.g.~\cite{IJCAI16}, which are similar to Datalog proof trees.
  Here, it is a consequence of the syntactic shape of $\Tmc'$ that
  such a proof tree for $N(a)$, with $a$ some individual, will not
  contain facts $A(a)$ (or ($r(a,b)$) in the derivation, where $A$ is
  not of the form $B^x$ (or $r$ not of the form $s^y$). It follows
  that $q^- \subseteq_{\Tmc'} q_0'$, a contradiction to minimality of
  $q$. Second, assume that $q$ contains a concept atom $A^x(y)$ (or
  role atom $r^x(y,z)$), with $y$ not the answer variable. Again
  regarding a proof tree for $N(a)$, the syntactic shape of $\Tmc'$
  prevents atoms of the described shape to occur at any other
  individual than $a$. It follows that $q^- \subseteq_{\Tmc'} q_0'$, again
  contradicting minimality of $q$.  \qed
\end{proof}

\PROPconformant*
\begin{proof}
  ``$\Rightarrow$''.  Assume that $Q$ is FO-rewritable. By
  Lemma~\ref{lem:french}, there is a UCQ-rewriting $q$ of $Q$ in which
  every CQ is a derivative of $q_0$. Let $q'$ be the corresponding UCQ
  for $Q'$. We argue that $q'$ is a rewriting of $Q'$. Thus let \Amc
  be a $\Sigma'$-ABox and $a_0 \in \mn{Ind}(\Amc)$. We have to show
  that $\Amc \models Q'(a_0)$ iff $\Amc \models q'(a_0)$. Let
  $\xbf=x_1 \cdots x_n$ be the answer variables in $q_0$ and let $\abf
  = a_1 \cdots a_n$ be a tuple of individual names that do not occur
  in \Amc. 

  In the first step, we unravel $\Amc$ into an infinite tree-shaped
  $\Sigma'$-ABox $\Amc'$ such that
  \begin{enumerate}

  \item $\Amc \models Q'(a_0)$ iff $\Amc' \models Q'(a_0)$ and 

  \item $\Amc \models q'(a_0)$ iff $\Amc' \models q'(a_0)$.

  \end{enumerate}
  A \emph{path in \Amc} is a sequence $b_1,r_1,b_2,r_2,\dots,b_k$ such
  that $b_1 = a_0$, $b_2,\dots,b_k \in \mn{Ind}(\Amc)$,
  $r_1,\dots,r_{k-1}$ are role names that occur in \Amc, and
  $r_i(b_i,b_{i+1}) \in \Amc$ for $1 \leq i < k$. $\Amc'$ consists of
  all assertions
  \begin{itemize}

  \item $A(p)$ whenever $p$ is a path in \Amc that
    ends with $b$ and $A(b) \in \Amc$; and

  \item $r(p,p')$ whenever $p$ is a path in \Amc that
    ends with $b$, $p'=rb'$ is a path in \Amc, and $r(b,b') \in \Amc$.

  \end{itemize}
  It follows from standard results about OMQs from $(\EL,\text{AQ})$
  that Condition~1 is satisfied, see for example~\cite{lutz-2012}. It can be 
  verified that Condition~2 is also satisfied since $q'$ is a tUCQ.

  \medskip

  Using compactness and monotonicity, it is easy to show that since
  $q$ is a rewriting of $Q$ on (finite) ABoxes, it is also a rewriting
  of $Q$ on infinite ABoxes. It thus remains to show that 
  \begin{enumerate}

  \item[3.] $\Amc' \models Q'(a_0)$ iff $\Bmc' \models Q(\abf)$ and

  \item[4.] $\Amc' \models q'(a_0)$ iff $\Bmc' \models q(\abf)$.

  \end{enumerate}
  where $\Bmc'$ is the (infinite) $\Sigma$-ABox that corresponds to
  $\Amc'$, that is, $\Bmc'$ is obtained from $\Amc'$ by replacing
  every assertion $A^{x_i}(a_0)$ with $A(a_i)$ and every assertion
  $r^{x_i}(a_0,b)$ with $r(a_i,b)$, adding $r(a_i,a_j)$ whenever
  $r(x_i,x_j) \in q_0$, and then removing all remaining assertions
  that contain a symbol from $\Sigma' \setminus \Sigma$ or the
  individual name $a_0$.

  \medskip

  In fact, Condition~3 can be shown using the construction of $Q'$ and
  by translating counter models, and Condition~4 can be shown using
  the construction of~$q'$. 

  \medskip ``$\Leftarrow$''. Assume that $Q'$ is FO-rewritable. Then
  there is a conformant tUCQ-rewriting $q'$ of $Q'$. Let $q$ be the
  corresponding UCQ for $Q$. We have to show that $q$ is a rewriting
  of $Q$ (this also establishes the ``moreover'' part of the
  lemma). Thus let \Amc be a $\Sigma$-ABox and $\abf \subseteq
  \mn{Ind}(\Amc)$. We aim to show that $\Amc \models Q(\abf)$ iff
  $\Amc \models q(\abf)$. Let $\xbf=x_1 \cdots x_n$ be the answer
  variables in $q_0$ and $\abf = a_1 \cdots a_n$. It suffices to
  consider ABoxes \Amc such that
  \begin{itemize}

  \item[($*$)] $r(x_i,x_j) \in q_0$ implies $r(a_i,a_j) \in \Amc$

  \end{itemize}
  since, otherwise, $\Amc \not \models Q(\abf)$ and $\Amc \not 
  \models q(\abf)$. 

  \medskip

  In the first step, we unravel $\Amc$ into an infinite $\Sigma$-ABox
  $\Amc'$ of more regular shape and with $\abf \subseteq
  \mn{Ind}(\Amc')$ such that
  \begin{enumerate}

  \item $\Amc \models Q(\abf)$ iff $\Amc' \models Q(\abf)$ and 

  \item $\Amc \models q(\abf)$ iff $\Amc' \models q(\abf)$.

  \end{enumerate}
  $\Amc'$ is constructed as follows. Start with the minimal ABox
  $\Amc'$ that satisfies ($*$). Then extend $\Amc'$ as follows. A
  \emph{path in \Amc} is a sequence $b_1,r_1,b_2,r_2,\dots,b_k$ such
  that $b_1 \in \abf$, $b_2,\dots,b_k \in \mn{Ind}(\Amc)$,
  $r_1,\dots,r_{k-1}$ are role names that occur in \Amc, and
  $r_i(b_i,b_{i+1}) \in \Amc$ for $1 \leq i < k$. Include in $\Amc'$
  all assertions
  \begin{itemize}

  \item $A(p)$ whenever $p$ is a path in \Amc that
    ends with $b$ and $A(b) \in \Amc$;

  \item $r(p,p')$ whenever $p$ is a path in \Amc that
    ends with $b$, $p'=rb'$ is a path in \Amc, and $r(b,b') \in \Amc$.

  \end{itemize}
  This finishes the construction of $\Amc'$. It can be proved that
  Conditions~1 and~2 are satisfied, translating counter models to
  prove Condition~1 and exploiting the construction of $q$ (which
  ensures that $q$ is a union of tqCQs that contains only role atoms
  from $q_0$ among the answer variables) for Condition~2.
  
  \medskip

  By compactness and monotonicity, $q'$ is a rewriting of $Q'$ also on
  infinite ABoxes. It thus remains to show that 
  \begin{enumerate}

  \item[3.] $\Amc' \models Q(\abf)$ iff $\Bmc' \models Q'(a_0)$ and

  \item[4.] $\Amc' \models q(\abf)$ iff $\Bmc' \models q'(a_0)$.

  \end{enumerate}
  where $\Bmc'$ is the (infinite) $\Sigma'$-ABox that corresponds to
  $\Amc'$, that is, $\Bmc'$ is obtained from $\Amc'$ by
  replacing all assertions $A(a_i)$ with $A^{x_i}(a_0)$ and all
  assertions $r(a_i,b)$ with $r^{x_i}(a_0,b)$, and then removing all
  role assertions that involve only answer variables.
  
  \medskip

  In fact, Condition~3 can be shown using the construction of $Q'$ and
  by translating counter models, and Condition~4 can be shown using
  the construction of $q$. In both cases, one exploits that ($*$)
  holds for $\Amc'$, which is a consequence of the fact that it
  holds for \Amc.
\qed
\end{proof}

\section{Proofs for Section~\ref{sect:rCQs}}

We introduce a backwards chaining algorithm for computing
UCQ-rewritings of OMQs from $(\EL,\text{rCQ})$ that we refer to as
$\mn{bc}_{\text{rCQ}}$. In a sense, $\mn{bc}_{\text{rCQ}}$ is the
natural generalization of $\mn{bc}_{\text{AQ}}$ to rCQs. We first need
to generalize some relevant notions underlying $\mn{bc}_{\text{AQ}}$.

Let $q$ be a CQ, $q' \subseteq q$, and $r(x,y) \in q$. Then $q'$ is a
\emph{tree subquery in $q$ with link $r(x,y)$} if $q'$ is tree-shaped and
the restriction of $q$ to the variables reachable from $y$ in the
directed graph $G_q$, $\mn{var}(q') \cap \mn{avar}(q)=\emptyset$, and
$s(u,z) \in q$ with $u \notin \mn{var}(q')$ and $z \in \mn{var}(q')$
implies $s(u,z)=r(x,y)$.  Note that, taken together, $r(x,y)$ and $q'$
can be viewed as an \EL-concept $\exists r . q'$. Let $q$ and $q'$ be
CQs, $C~\sqsubseteq~D$ a concept inclusion, and $x \in
\mn{var}(q)$. Then $q'$ is \emph{obtained from $q$ by applying $C
  \sqsubseteq D$ at $x$} if $q'$ can be obtained from $q$ by
\begin{itemize}

\item removing $A(x)$ for all concept names $A$ with $\models D
  \sqsubseteq A$;

\item for each tree subquery $q'$ of $q$ with link $r(x,y)$ such that
  $\models D \sqsubseteq \exists r . {q'}$, removing $r(x,y)$
  and~$q'$;

\item adding $A(x)$ for all concept names $A$ that occur in $C$ as a
  top-level conjunct;

\item adding $\exists r . E$ as a CQ with root $x$, for each
  $\exists r.E$ that is a top-level conjunct of~$C$.

\end{itemize}
%
%
Let $q,q'$ be CQs. We write $q' \prec q$ if $q'$ can be obtained from
$q$ by selecting a tree subquery $q''$ in $q$ with link $r(x,y)$ and
removing both $r(x,y)$ and $q''$. We use $\prec^\ast$ to denote the
reflexive and transitive closure of~$\prec$ and say that \emph{$q'$ is
  $\prec$-minimal with $q' \subseteq_\Tmc q_0$} if $q'
\subseteq_\Tmc q_0$ and there is no $p \prec q'$ with $\Tmc
\models p \sqsubseteq q_0$.

Started on OMQ $Q=(\Tmc\!,\Sigma,q_0)$, algorithm
$\mn{bc}_{\text{rCQ}}$ starts with a set $R$ that contains for each
fork rewriting $q_r$ of $q_0$ a CQ $p \prec^\ast q_r$ that is
$\prec$-minimal with $p \subseteq_\Tmc q_0$ and then
exhaustively performs the same steps as $\mn{bc}_{\text{AQ}}$:
\begin{enumerate}

  \item find $q \in R$, $x \in \mn{var}(q)$, a concept inclusion
  $E \sqsubseteq F \in \Tmc$, and $q'$ such that $q'$ is obtained from $q$ by applying 
  $E \sqsubseteq F$ at $x$;

  \item find a $q'' \prec^\ast q'$ that is $\prec$-minimal with $q'' 
  \subseteq_\Tmc q_0$, and add $q''$ to $R$.

\end{enumerate}
We use $\mn{bc}_\text{rCQ}(Q)$ to denote the potentially infinitary
UCQ $\bigvee R|_\Sigma$, $R$ obtained in the limit.

The following establishes the central properties of the
$\mn{bc}_{\text{rCQ}}$ algorithm. It is proved by showing that there
is a correspondence between the backwards chaining implemented in
$\mn{bc}_{\text{rCQ}}$ and the chase, a forward chaining procedure
that can be applied to an ABox \Amc and a TBox \Tmc to construct a
universal model of \Amc and \Tmc\!\!\!, that is a model that gives exactly
the certain answers on \Amc to any OMQ from $(\EL,\text{AQ})$ based on
$\Tmc$\!.
%
%
\begin{lemma}
\label{lem:soundcompl}
  Let $Q=(\Tmc\!,\Sigma,q_0)$ be an OMQ from $(\EL,\text{rCQ})$.
  If $\mn{bc}_{\text{rCQ}}(Q)$ is finite, then it is a UCQ-rewriting of $Q$.
  Otherwise, $Q$ is not FO-rewritable.
%
%
%
\end{lemma}
In preparation for the proof of Lemma~\ref{lem:soundcompl}, we remind
the reader of the standard chase procedure. The chase is a forward
chaining procedure that exhaustively applies the concept 
inclusions of a TBox to an ABox in a rule-like fashion. Its final
result is a (potentially infinite) ABox in which all consequences of
\Tmc are materialized. To describe the procedure in detail, it is
helpful to regard \EL-concepts $C$ as tree-shaped ABoxes
$\Amc_{C}$. $\Amc_{C}$ can be obtained from the concept query
corresponding to $C$ by identifying its individual variables with
individual names. Now let \Tmc be an \EL-TBox and \Amc an ABox.
A \emph{
chase step} consists in choosing a concept inclusion $C \sqsubseteq D \in
\Tmc$ and an individual $a \in \mn{Ind}(\Amc)$ such that $\Amc \models
C(a)$, and then extending \Amc by taking a copy $\Amc_{D}$ of $D$
viewed as an ABox with root $a$ and such that all non-roots are fresh
individuals, and then setting $\Amc := \Amc \cup \Amc_{D}$.  
The \emph{result of chasing \Amc with \Tmc}\!,
denoted with $\mn{ch}_\Tmc(\Amc)$, is the ABox obtained by
exhaustively applying chase steps to \Amc in a fair way. It is
standard to show that the chase procudes a universal model, i.e.\ for
all CQs $q$ and tuples $\abf=(a_1,\dotsc,a_n)$ over
$\mn{Ind}(\Amc)$, it holds that $\Amc,\Tmc \models q(\abf)$ iff
$\mn{ch}_\Tmc(\Amc) \models q(\abf)$.
%
We now prove Lemma~\ref{lem:soundcompl}.
\begin{proof}
  For the first part, let $\mn{bc}_{\text{rCQ}}(Q)$ be finite, and \Amc be a 
  $\Sigma$-ABox and $\abf \subseteq 
  \mn{Ind}(\Amc)$. We have to show 
  $\Amc \models \bigvee\!R|_\Sigma(\abf)$ iff $\Amc 
  \models Q(\abf)$. For direction 
  ``$\Rightarrow$'', assume that $\Amc \models 
  \bigvee\!R|_\Sigma(\abf)$. Then there 
  is a $q \in R|_\Sigma$ with $\Amc \models q(\abf)$. 
  Consequently $\Amc,\Tmc \models 
  q(\abf)$. By construction of $R$, all its elements $q$ satisfy $q 
  \subseteq_\Tmc q_0$, thus $\Amc \models Q(\abf)$, as required. 
  
  For direction ``$\Leftarrow$'', we examine the chase sequence that 
  witnesses $\Amc \models Q(\abf)$. W.l.o.g.\ we can 
  assume that \Tmc contains no conjunctions on the right-hand side of concept 
  inclusions, i.e.\ \Tmc consists only of concept inclusions of the form $C 
  \sqsubseteq A$ and $C \sqsubseteq \exists r.D$. 
  If $\Amc \models Q(\abf)$, then $\mn{ch}_\Tmc(\Amc) \models q_0(\abf)$ and 
  consequently, there is a sequence of (not necessarily $\Sigma$-) ABoxes 
  $\Amc=\Amc_0,\Amc_1,\dots,\Amc_k$ that \emph{demonstrates} 
  $\mn{ch}_\Tmc(\Amc) \models q_0(\abf)$, that is, each $\Amc_{i+1}$ 
  is obtained from $\Amc_i$ by a single chase step and $\Amc_k \models 
  q_0(\abf)$. 
  
  It thus suffices to prove by induction on $k$ that
  if 
  $\Amc=\Amc_0,\dots,\Amc_k$ is a chase sequence that demonstrates
  $\mn{ch}_\Tmc(\Amc) \models q_0(\abf)$, then $\Amc \models 
  \bigvee\!R|_\Sigma(\abf)$.
  The induction start is trivial: For $k=0$, $\Amc_k \models q_0(\abf)$
  implies $\Amc \models q_0(\abf)$. Since $q_0$ is a 
  fork rewriting of itself, and by definition of $R_0$, there is a query $p 
  \prec^\ast q_0$ in $R$. Restrict the homomorphism witnessing $\Amc\models 
  q_0(\abf)$ to the variables still present in $p$, and the result is a 
  homomorphism from $p$ to \Amc, mapping the answer variables to \abf. Thus, 
  we have $\Amc \models p(\abf)$, and $\Amc \models 
  \bigvee\!R|_\Sigma(\abf)$.
  For the induction step, assume that
  $\Amc=\Amc_0,\dots,\Amc_k$ is a chase sequence that demonstrates
  $\mn{ch}_\Tmc(\Amc) \models q_0(\abf)$, with $k>0$.  Applying IH to 
  the subsequence $\Amc_1,\dots,\Amc_k$, we obtain that $\Amc_1 \models
  \bigvee\!R|_\Sigma(\abf)$.  Thus there is a $q \in 
  R|_\Sigma$ with $\Amc_1 \models
  q(\abf)$, witnessed by a homomorphism $h$ from $q$ to $\Amc_1$ with $h(\xbf) 
  = \abf$. If $\Amc_0 \models q(\abf)$, then we are done. Otherwise, look at 
  the chase step that led from $\Amc_0$ to $\Amc_1$. 
%
%
%
  Assume that $\Amc_1$ is 
  obtained from $\Amc_{0}$ by choosing a concept inclusion $E \sqsubseteq F 
  \in 
  \Tmc$ and $b \in \mn{Ind}(\Amc_{0})$ with $\Amc_{0}\models E(b)$, and adding 
  a copy of the ABox $\Amc_{F}$ to $\Amc_{0}$ at $b$. There are two 
  possibilities:
  \begin{itemize} 
  
  \item $F = A$ 
  
  An atom $A(b)$ is added to $\Amc_0$. Further, let $z_1,\dots,z_n$ be 
  all variables of $q$ such that $h(z_i)=b$ and $A(z_i)\in q$. There must be 
  at least one such $z_i$, since otherwise $h$ would not depend on any 
  assertions added in the construction of $\Amc_1$ from $\Amc_0$ and thus 
  witnessed $\Amc_{0}\models q(\abf)$, a contradiction.
  \medskip 
  
  \item $F = \exists r.G$
    
  An atom $r(b,d)$ with $d$ fresh, and an ABox $\Amc_G$ of fresh individuals 
  rooted in $d$, are added to $\Amc_0$. Let $\ybf = y_1,\dotsc,y_m$ be all 
  variables of $q$ such that $h(y_i)=d$, and $\zbf = z_1,\dots,z_n$ be all 
  variables of $q$ such that $h(z_i)=b$ and there is at least one 
  $r$-successor of $z_i$ in $\ybf$. As above, there must be at least one such 
  $z_i$. Note that there are no answer variables among $\ybf$, as $d$ is 
  anonymous.
  \end{itemize}
  Ideally, in the second case we would have $q$ conform to the following 
  property:
  \begin{itemize}
    \item[($*$)] For every $y_i\in\ybf$, it holds that $y_i$ is the root of a 
    tree subquery $q'_i$ of $q$ with link $r(z_j,y_i)$, and $\Amc_G 
    \models q'_i$, where $z_j\in\zbf$.
  \end{itemize}
  Note that this is not guaranteed, as there might be forks $r(z_i,y), 
  r(z_j,y)$ occurring in $q$ at $\zbf$, or $r(y_i, x), r(y_j,x)$ with 
  $y_i,y_j$ from $\ybf$ or below. These variables 
  could still be mapped to the tree-shaped part $\Amc_F(b)$ of $\Amc_1$ by 
  $h$. Nonetheless, we can find a query $q' \in R$ such that $q'$ fulfills 
  ($*$): Assume there is a fork $r(z_i,y), r(z_j,y)$ 
  with $y \in \ybf$ (forks below are handled in the same way), and $q$ is a 
  derivative of some $\prec$-minimized fork rewriting $p$ of $q_0$. Then 
  variables $y, z_i, z_j$ are part of the core of $p$, as backward application 
  of concept inclusions does not generate forks. Let $p'$ be $p$ with the fork 
  $r(z_i,y), r(z_j,y)$ eliminated; we are guaranteed to have a query $p'' = 
  \mn{min}(p') \in R$. Note first that in $p''$, the subtree rooted in $z$ 
  (the identification of $z_i$ and $z_j$) might have been deleted by 
  minimization. There has to be another $r(z_k,y_k)$, as otherwise, 
  $\Amc_0\models q(\abf)$, in which case we again would be done. From $p''$, 
  we can obtain a derivative $q'$ of $p''$ of the desired form by backwards 
  application of concept inclusions in $\Tmc$\!. This can be shown by 
  induction on 
  the length of the backwards chaining sequence that led from $p$ to $q$: 
  Either we can apply a concept inclusion $\alpha$ to both $p$ and $p''$, or 
  it is applied to the deleted tree in $p$, in which case we omit this 
  application in derivatives of $p''$ when generating $q'$. In either case, 
  the original fork will not be present in $q'$. Continue the proof using 
  query $q'$ for $q$.
  \smallskip
  
  Let the CQs $q^0,\dots,q^n$ be such that $q=q^0$, and $q^{i+1}$ can be
  obtained from $q^i$ by doing the following if $z_{i} \in
  \textup{var}(q^i)$ (otherwise, just set $q^{i+1} := q^i$):
  \begin{enumerate}
  
  \item remove $A(z_i)$ if $F = A$;
  
  \item remove $r(z_i,y)$ and the subquery $q'$ of $q$ with link $r(z_i,y)$ if 
  $\models F \sqsubseteq \exists r.q'$;
  
  \item add $A(z_i)$ for all concept names $A$ that are top-level conjuncts of 
  $E$;
  
\item add $\exists r.H$ as a CQ with root $z_i$, for each
  $\exists r.H$ that is a top-level conjunct of
  $E$; 
  
  
  \item minimize the resulting ${q^i}'$, that is, choose
    $q^{i+1} \prec^{\ast} {q^i}'$ such that
    $q^{i+1}$ is $\prec$-minimal with $q^{i+1} \subseteq_\Tmc q_0$.
  
  \end{enumerate}
  It is easy to prove by induction on $i$ that $q^i \in R$ for all $i \leq n$. 
  It thus remains to argue that $\Amc_0 \models q^n(\abf)$. To do this, we 
  produce maps $h_0,\dots,h_n$ such that $h_i$ is a homomorphism from $q^i$ to 
  $\Amc_1$ with $h_i(\xbf) = \abf$ and such that $h_i(z_j)=b$ if $z_j \in 
  \textup{var}(q^i)$, for all $i \leq n$.  Start with $h_0=h$. To produce 
  $h_{i+1}$ from $h_i$, first restrict $h_i$ to the remainder of $q^i$ after 
  the removals in Step~2 were carried out. Then extend $h_i$ to cover all fresh
  elements introduced via the subtrees $\exists r .H$ in Step~4. Note that, 
  since $\Amc_0 \models E(b)$ and $h_i(z_i)=b$, this is possible. For the same 
  reason, the resulting homomorphism $h_{i}'$ respects all the concept 
  assertions added in Step~3. Finally, to deal with the minimization in 
  Step~5, restrict $h_{i}'$ to $\mn{var}(q^{i+1})$.
  
  By construction of the queries $q^0,\dots,q^n$ and the homomorphisms $h_0$, 
  $\dots$, $h_n$, there is no atom in $q^n$ such that the image of the atom 
  under $h_n$ is in $\Amc_1 \setminus \Amc_0$. To show this, assume to the 
  contrary that there is such an atom $A(x)$ 
  in $q^n$. There are two cases:
  \begin{enumerate}
  
  \item $h(x)=b$.
  
    Then $x=z_i$ for some $i$. Since $A(h(x))=A(b)$ was added by the   
    application of $E \sqsubseteq A$, the atom $A(x)$ was removed in Step~1 
    when constructing $q^{i+1}$ from $q^i$, in contradiction to $A(x)$ being 
    in $q^n$.
    \smallskip
    
  \item $h(x) \neq b$.
  
    Then $h(x)$ is a non-root node of the sub-ABox $\Amc_{\exists r.G}(b)$ of 
    $\Amc_{1}$, and $F = \exists r.G$. There is an answer variable 
    $x_j\in\xbf$ such that there is a path from $x_j$ to $x$, i.e.\ a sequence 
    of individuals $x_j=y_0,\dotsc,y_\ell=x$ such that $r_i(y_i,y_{i+1}) \in 
    q^n$ for some $r_i$, for all $i < \ell$. We find a corresponding path 
    $h(y_0),\dots, h(y_\ell)$ in $\Amc_1$, and since $\Amc_G$ has been linked 
    to $\Amc_0$ only by $r(b,d)$, the individual $b$ must be on that second 
    path. Let $y_p$ be such that $h(y_p)=b$. We must have $y_p= z_i$ for some 
    $i$, and $r_{p}=r$. Note that by ($*$), $y_{p+1}$ is the root of a tree 
    subquery $q'$ of $q^n$ with link $r(z_i,y_{p+1})$ (recall that links are 
    unique). Homomorphism $h_n$ maps $y_{p+1}$ to $d$, so we have $\models F 
    \sqsubseteq \exists r.q'$. Consequently, the subtree of $q^n$ rooted at 
    $y_{p+1}$ was removed in Step~2 when constructing $q^{i+1}$, in 
    contradiction to $A(x)$ being in $q^n$.
  
  \end{enumerate}
  The case of role atoms is similar to subcase~2 above, but simpler (we know 
  that $F = \exists r.G$). We have thus shown that there is no atom in $q^n$ 
  such that the image of this atom under $h_n$ is in $\Amc_1 \setminus 
  \Amc_0$. Consequently $\Amc_0 \models q^n(\abf)$ via $h_n$. 
  As $\Amc_0$ is a $\Sigma$-ABox, it holds that $q^n 
  \in R|_\Sigma$, and we are done with part~1.
  \bigskip

  \noindent 
  Now for the second part of Lemma~\ref{lem:soundcompl}. We prove the 
  contrapositive, using a result from \cite{IJCAI16}:
  \\[2mm]
  {\bf Fact.}  Let $Q=(\Tmc\!,\Sigma,q_0)$ be an OMQ from $(\EL,
  \text{rCQ})$. $Q$ is FO-rewritable iff there is a $k \geq
  0$ such that for all pseudo ditree $\Sigma$-ABoxes
  \Amc 
  of outdegree at most~$|\Tmc|$ and width at most $|q|$: if $\Amc
  \models Q(\abf)$ with $\abf$ from the core of $\Amc$, then
  $\Amc|_{\leq k} \models Q(\abf)$.
  \\[2mm]
  We refrain from giving a detailed definition of the notions used in
  the above statement and only mention that, informally, a pseudo
  ditree ABox \Amc of width $i$ is a tree-shaped ABox (with all edges
  pointing downwards and without multi-edges) whose root has been
  replaced by an ABox with at most $i$ individuals, called the core of
  \Amc. The outdegree refers to the non-core part of \Amc, and
  $\Amc|_{\leq k}$ means the result of removing all nodes from the
  tree part of \Amc that are of depth exceeding $k$, that is, that are
  more than $k$ steps away from the core. 
  
  It is straightforward to verify that every query $q$ ever added to  
  $\mn{bc}_{\text{rCQ}}(Q)$, viewed as an ABox $\Amc_q$, is a pseudo ditree
  $\Sigma$-ABox of width at most $|q|$ such that $\Amc_q \models
  Q(\abf)$ where \abf are the individuals in $\Amc_q$ that correspond
  to the answer variables in $q$. Using that $q$ is $\prec$-minimal
  with $q \subseteq_\Tmc q_0$, it can be shown that removing any
  subtree from $\Amc_q$ results in an ABox $\Amc'$ with $\Amc'
  \not\models Q(\abf)$. We say that $\Amc_q$ is \emph{$\prec$-minimal
    with } $\Amc_q \models Q(\abf)$. This, in turn, can be used to
  prove in a standard way that $\Amc_q$ has outdegree at most
  $|\Tmc|$. By the above fact and the $\prec$-minimality of $\Amc_q$,
  the depth of $\Amc_q$ is thus at most $k$. Clearly, there are only finitely 
  many pseudo ditree $\Sigma$-ABoxes of bounded width, outdegree and depth. 
\qed
\end{proof}
\medskip

\noindent
We next prove Proposition~\ref{lem:red2-correct}, starting with some
preliminaries.  Let $Q= (\Tmc\!, \Sigma, q_0)$ be an OMQ from $(\EL,
\text{rCQ})$ and $q_r$ a fork rewriting of $q_0$.

A conformant tCQ $q'$ can be converted into a corresponding CQ $q$ for
$Q$, as detailed before Proposition~\ref{lem:red2-correct}. For easier
reference, we use $\pi(q')$ to denote $q$. Conversely, let $q$ be a
derivative of $q_r$ in the sense that $q$ can be obtained from the
restriction of $q_r$ to the variables in $\mn{core}(q_r)$ by adding
tree-shaped CQs rooted at variables in $\mn{core}(q_r)$. We can translate $q$
into a \emph{corresponding CQ $q'$ for $Q_{q_r}$} as follows: replace
every atom $A(x)$, $x \in \mn{core}(p)$, with $A^x(x_0)$; every atom
$r(x,y)$, $x \in \mn{core}(p)$ and $y \notin \mn{core}(p)$, with
$r^x(x_0,y)$; delete all atoms $r(x_1,x_2)$, $x_1,x_2 \in
\mn{core}(p)$.  The answer variable in $q'$ is $x_0$. We use $\tau(q)$
to denote the query $q'$. 

It can be verified that $\tau$ produces conformant tUCQs, and that
$\pi(\tau(q))=q$. Moreover, both $\pi$ and $\pi^-$ are injective,
$\pi$ translates $\Sigma'$-queries into $\Sigma$-queries, and $\tau$
translates $\Sigma$-queries into $\Sigma'$-queries.

When $q$ is a derivative of $q_r$, then $p \prec^*q$ is a
\emph{$\prec$-minimization} of $q$ if $p$ is minimal with $p
\subseteq_\Tmc q_0$. Note that this is exactly the minimization
carried out in Step~2 of the $\mn{bc}_{\text{rCQ}}$ algorithm started
on $Q$. When $q'$ is a conformant tCQ, then $p' \prec^*q'$ is a
\emph{$\prec$-minimization} of $q'$ if $p'$ is minimal with
$\Tmc^{\mn{min}}_{q_r} \models p' \sqsubseteq N$. Note that this is
exactly the minimization carried out in Step~2 of the
$\mn{bc}^+_{\text{AQ}}$ algorithm started on
$(Q_{q_r},\Tmc^{\mn{min}}_{q_r})$.
\begin{lemma}
\label{lem:minlem}
  Let $Q= (\Tmc\!, \Sigma, q_0)$ be an OMQ from $(\EL, \text{rCQ})$
  and $q_r$ a fork rewriting of $q_0$.
  Then 
  \begin{enumerate}

  \item if $q$ is a derivative of $q_r$ and $p$ a $\prec$-minimization 
    of $q$, then $\tau(p)$ is a $\prec$-minimization of $\tau(q)$;

  \item if $q'$ is a conformant tCQ and $p$ a $\prec$-minimization
    of $\pi(q')$, then there is a $\prec$-minimization $p'$ of $q'$
    with $\pi(p')=p$;

  \item if $q'$ is a conformant tCQ and $p'$ a $\prec$-minimization of
    $q'$, then $\pi(p')$ is a $\prec$-minimization of $\pi(q')$.

  \end{enumerate}
\end{lemma}
\begin{proof}
We only sketch a proof of Point~1, Points~2 and~3 are established very
similarly.

Recall that $\prec$-minimization of $\tau(q)$ is based on the TBox
$\Tmc^{\mn{min}}_{q_r}$ and that, due to Lemma~\ref{lem:splitting},
for any query $q' \prec^\ast \tau(q)$ it holds that
$\Tmc^{\mn{min}}_{q_r} \models \bigsqcap_{C(x) \text{ a tree in }
  q'} C_R^x \sqsubseteq N$ iff $\Amc_{q'},\Tmc \models
q_0(\abf)$ for some tuple $\abf$ that can be obtained by starting with
the tuple \xbf of answer variables in $q_0$, then potentially
replacing each variable $x$ from \xbf with a variables from
$[x]_{q'}$, and finally replacing each variable with the
corresponding individual name in $\Amc_{q'}$. It is standard to
prove that this, in turn, is the case iff $q' \subseteq_\Tmc q_0$,
which is what minimization of $q$ is based on.
\qed
\end{proof}

\LEMredTwocorrect*

\begin{proof}
  Let $Q = (\Tmc\!, \Sigma, q_0)$ be an OMQ from
  $(\EL, \text{rCQ})$. We prove Proposition~\ref{lem:red2-correct} by
  relating the runs of $\mn{bc}^+_{\text{AQ}}(q_r)$, $q_r$ a fork
  rewriting of $q_0$, with the run of $\mn{bc}_{\text{rCQ}}(Q)$.

  Recall that $\mn{bc}_{\text{rCQ}}(Q)$ initializes $R$ with a set
  that contains for each fork rewriting $q_r$ of $q_0$, a CQ $p
  \prec^\ast q_r$ that is $\prec$-minimal. For easier reference, we
  denote this initial set $R$ with $R_0$. We use $\mn{min}(q_r)$ to
  denote $p$, thus $R_0 = \{ \mn{min}(q_r) \mid q_r \text{ fork
    rewriting of } q_0 \}$.  Note that the different queries in $R_0$
  do not interact during the run of $\mn{bc}_{\text{rCQ}}(Q)$, that
  is, the final set $R$ can be written as $\bigcup_{p \in R_0} R_p$
  where $R_p$ denotes the result of starting with the set $R= \{ p \}$
  and then exhaustively applying Steps~1 and~2 of
  $\mn{bc}_{\text{rCQ}}$.

  It follows from Point~1 of Lemma~\ref{lem:minlem} that whenever
  $\mn{min}(q_r)=\mn{min}(q'_r)$ for two fork rewritings $q_r$ and
  $q'_r$ of $q_0$, then we can guide the very first minimization
  (after replacing the concept name $N$) during the runs of
  $\mn{bc}^+_{\text{AQ}}(Q_{q_r},\Tmc_{q_r}^{\mn{min}})$ and
  $\mn{bc}^+_{\text{AQ}}(Q_{q'_r},\Tmc_{q'_r}^{\mn{min}})$ (which
  involve `don't care non-determinism') such that, in both cases, they
  query $\tau(p)$ is added to $M$. Consequently, the sets $M$ computed
  in the limit are identical. It therefore suffices to consider for
  each $p \in R_0$ one run
  $\mn{bc}^+_{\text{AQ}}(Q_{q_r},\Tmc_{q_r}^{\mn{min}})$ such that
  $\mn{min}(q_r)=p$. We use $M_p$ to denote the (finite or infinite)
  set of CQs $M$ generated by such a run. 
  Our main aim is
  to show the following.
\\[2mm]
{\bf Claim.} For each $p \in R_0$, $R_p = \{ \pi(q')
\mid q' \in M_p \}$.
\\[2mm]
We argue that this establishes Proposition~\ref{lem:red2-correct} and
then prove the claim. 

First let $Q$ be FO-rewritable. Assume to the contrary of what we have
to show that $\mn{bc}^+_{\text{AQ}}(Q_{q_r},\Tmc_{q_r}^{\mn{min}})$ is
infinite for some fork rewriting $q_r$. Let $\mn{min}(q_r)=p$. Then the
set $M_p$ contains infinitely many $\Sigma'$-queries, thus $R_p$
contains infinitely many $\Sigma$-queries by injectivity of
$\pi$. By Lemma~\ref{lem:soundcompl}, this means that $Q$ is not
FO-rewritable, a contradiction. We also have to show that
$$
\displaystyle
\bigvee_{p \text{ fork rewriting for } q_0} \;\; \bigvee_{q' \in M_p|_\Sigma}
\pi(q') 
$$
is a rewriting of $Q$. However, by the claim the above query is simply
$\bigvee R|_\Sigma$, $R$ the set computed by
$\mn{bc}_{\text{rCQ}}(Q)$.  It thus suffices to invoke
Lemma~\ref{lem:soundcompl}. 

Conversely, let $Q$ be non-FO-rewritable. By
Lemma~\ref{lem:soundcompl}, $\mn{bc}_{\text{rCQ}}(Q)$ is infinite and
thus $R_p|_\Sigma$ is infinite for at least one $p$. By injectivity of
$\pi^-$, $M_p|_\Sigma$ is infinite.

  %

  \smallskip
  We now prove the claim. Let $p \in R_0$. 

  \medskip 

  ``$\subseteq$''. Let $q\in R_p$. Then there is a sequence of CQs
  $p=q_1, \dotsc, q_m = q$ such that each $q_{i+1}$ is obtained from
  $q_i$ by applying Steps~1 and~2 of the $\mn{bc}_{\text{rCQ}}$
  algorithm. We prove by induction on $i$ that $\tau(q_i) \in M_p$,
  thus $q=\pi(\tau(q)) \in \{ \pi(q') \mid q' \in M_p \}$ as
  required.
  
  For the induction start, note that
  $\mn{bc}^+_{\text{AQ}}(Q_{q_r},\Tmc^{\mn{min}}_{q_r})$ initializes
  $M_p$ with $\{ N(x) \}$, that $\bigsqcap_{C(x) \text{ a tree in }
    q_r} C_R^x \sqsubseteq N$ is in $\Tmc_{q_r}$, and that the
  left-hand side of this CI is nothing but $\tau(q_r)$. Thus,
  $\mn{bc}^+_{\text{AQ}}(Q_{q_r},\Tmc^{\mn{min}}_{q_r})$ adds a
  minimization of $q_r$ to $M$. By Point~1 of Lemma~\ref{lem:minlem}, we can
  assume this minimization to be $p$, thus $p \in M$ as required.  
 
  For the induction step, 
%
%
%
  assume that $q_{i+1}$ is obtained from $q_{i}$ by application 
  of a concept inclusion $\alpha = C \sqsubseteq D \in \Tmc$ at $x$, resulting 
  in a CQ $\widehat q$, and subsequent minimization of $\widehat q$ according 
  to Step~2 of the $\mn{bc}_{\text{rCQ}}$ algorithm. We first show that it is 
  possible to apply a concept inclusion $\alpha' \in \Tmc_{q_r}$ at a variable 
  $x'$ in $\tau(q_i)$ such that the resulting CQ is $\tau(\widehat q\kern .5pt)$. There
  are two cases:
  \begin{enumerate}

  \item $x \notin \mn{core}(p)$. Then $x$ and the subtree below it are
    present in $\tau(q_i)$. We apply $\alpha'=\alpha$ at $x'=x$ in
    $\tau(q_i)$. 
    
  \item $x \in \mn{core}(p)$. 
  %
  %
  %
  %
  %
  %
   We apply $C^x_R \sqsubseteq D^x_R \in \Tmc_{q_r}$ at $x_0$ in
   $\tau(q_i)$.

   %
 %
  %
%
\end{enumerate}
In both cases, it can be verified that the resulting query is
$\tau(\widehat q\kern 0.5pt)$.  It remains to apply Point~1 of 
Lemma~\ref{lem:minlem}:
since $\widehat q$ is a derivative of $q_r$ and $q_{i+1}$ is the
result of minimizing $\widehat q$, we can minimize $\tau(\widehat q\kern 
0.5pt)$
to obtain $\tau(q_{i+1})$.


  \medskip

  ``$\supseteq$''.  Let $q \in M_p$. Then $M_p$ contains a sequence of
  \EL-concepts $N(x)=q_0,\dotsc,q_m=q$ such that each $q_{i+1}$ is
  obtained from $q_i$ by applying Steps~1 and~2 of the
  $\mn{bc}^+_{\text{AQ}}$ algorithm, using TBox $\Tmc_{q_r}$ in Step~1 and
  $\Tmc^{\mn{min}}_{q_r}$ in Step~2. We prove by induction on $i$ that 
  $\pi(q_i)\in R_p$ for $1 \leq i \leq m$.
  
  For the induction start, note that $M_p$ is initialized to
  $\{N(x)\}$, and the concept inclusion $\bigsqcap_{C(x) \text{ a tree
      in } q_r} C_R^x \sqsubseteq N$ is in $\Tmc_{q_r}$; no
  other concept inclusion can be applied to $N(x)$. Consequently,
  Step~1 of the $\mn{bc}^+_\text{AQ}$ algorithm produces a query
  $\widehat{q}$ that is the left-hand side of this concept inclusion.
  By Point~2 of Lemma~\ref{lem:minlem}, we can
  assume w.l.o.g.\ that the minimization $q_0$ of $\widehat q$
  satisfies $\pi(q_0)=p$, thus $q_0 \in R_p$.

    
  For the induction step, assume that $q_{i+1}$ is obtained from $q_i$
  by application of a concept inclusion $\alpha \in \Tmc_{q_r}$ at
  $x$, yielding a query $\widehat q$, and subsequent minimization
  based on $\Tmc^\mn{min}_{q_r}$. We show that a concept inclusion
  $\alpha'\in \Tmc$ is applicable in $\pi(q_i)$ to yield
  $q_{i+1}$. There are two possibilities:
  \begin{enumerate}
    
    \item $\alpha$ is applied at a variable $x \neq x_0$ in $q_i$, thus is 
    not of the form $E^x_R \sqsubseteq F^x_R$, as $q_i$ is conformant. 
    Then $x$ and the subtree below it are present in $\pi(q_i)$. The concept 
    inclusion $\alpha$, which is present in \Tmc\!\!\!, can be applied at $x$ 
    in 
    $\pi(q_i)$, so $\alpha' = \alpha$.
    
    \item $\alpha = E^x_R \sqsubseteq F^x_R$ is applied at $x_0$ of $q_i$. 
    Application of $\alpha$ results in:
    \begin{itemize}
      \item[(a)] removal of atom $A^x(x_0)$ if $F^x_R = A^x$;
      
      \item[(b)] removal of all existential subtrees rooted at some $y$, 
      together with the atom $r(x_0,y)$, whenever $r(x_0,y)\in q_i$ and 
      $\models F \sqsubseteq \exists r.(q_i|_y)$;
      
      \item[(c)] adding of $A^x(x_0)$ if $F^x_R = A^x$;
      
      \item[(d)] adding of $r(x_0, y)$ and $G$ as a CQ with root $y$ if 
      $F^x_R = \exists r.G$.
      
    \end{itemize} 
    The concept inclusion $\alpha' = E \sqsubseteq F \in \Tmc$ is applicable 
    at $x$ in $\pi(q_i)$. Removal of atoms in (a) and (b) is done 
    correspondingly at variable $x$ in $\pi(q_i)$, the same holds 
    for the adding of concept names or subtrees in (c) and (d).
  \end{enumerate}
  In both cases, it can be verified that the resulting query is 
  $\pi(\widehat q\kern 0.5pt)$. It remains to observe that by Point~3 of
  Lemma~\ref{lem:minlem}, $\mn{bc}_\text{rCQ}$ can minimize
  $\pi(\widehat q\kern 0.5pt)$ to produce $\pi(q_{i+1})$.
\qed
\end{proof}

\end{document}